\documentclass{article}

\usepackage[utf8]{inputenc}
\usepackage[letterpaper]{geometry} 
\AtBeginDocument{
  \newgeometry{
    textheight=9in,
    top=1in,
    headheight=12pt,
    headsep=25pt,
    footskip=30pt
  }
}

\PassOptionsToPackage{hyphens}{url}
\usepackage[hidelinks]{hyperref}
\usepackage{graphicx}
\usepackage{amsfonts}
\usepackage{amsmath}
\usepackage{amsthm}
\usepackage{amssymb}
\usepackage{verbatim}
\usepackage{physics}
\usepackage{booktabs}
\usepackage{colortbl}
\usepackage{authblk}
\usepackage{aligned-overset}

\usepackage[framemethod=tikz]{mdframed}
\usetikzlibrary{shadows}
\definecolor{captiongray}{HTML}{555555}
\mdfsetup{%
  innertopmargin=2ex,
  innerbottommargin=1.8ex,
  linecolor=captiongray,
  linewidth=0.5pt,
  roundcorner=5pt,
  shadow=true,
  shadowcolor=black!05,
  shadowsize=4pt
}
\newenvironment{hoggfigure}{%
  \begin{figure}[tp]%
    \begin{mdframed}%
    \color{captiongray}}{%
    \end{mdframed}%
  \end{figure}}

\definecolor{rb4}{HTML}{27408B}

\newcommand{\todo}[1]{\textcolor{red}{#1}}

\theoremstyle{definition}
\newtheorem{definition}{Definition}

\newtheorem{theorem}{Theorem}
\newtheorem*{theorem*}{Theorem}
\newtheorem{proposition}{Proposition}
\newtheorem*{proposition*}{Proposition}
\newtheorem*{remark}{Remark}
\newtheorem{lemma}{Lemma}
\newtheorem*{properties}{Properties}

\newcommand{\tensorname}[2]{{#1}_{(#2)}}
\newcommand{\tensor}[2]{$\tensorname{#1}{#2}$-tensor}
\newcommand{\tensors}[2]{$\tensorname{#1}{#2}$-tensors}
\NewDocumentCommand\contract{mg}{%
    \ensuremath{\iota_{#1}\IfNoValueTF{#2}{}{\qty(#2)}}%
}
\DeclareMathOperator*{\argmax}{arg\,max} 

\newcommand{\sectionname}{Section}

\newcommand{\figref}[1]{\figurename~\ref{#1}}

\sloppypar
\allowdisplaybreaks

\title{\bfseries%
Equivariant geometric convolutions \\
for emulation of dynamical systems
}
\author[1]{Wilson~G.~Gregory}
\author[2,3,4]{David~W.~Hogg}
\author[1]{Ben~Blum-Smith}
\author[5]{Maria~Teresa~Arias}
\author[1]{Kaze~W.~K.~Wong}
\author[1,6]{Soledad~Villar}
\affil[1]{Department of Applied Mathematics and Statistics, Johns Hopkins University, Baltimore, MD, USA}
\affil[2]{Center for Cosmology and Particle Physics, Department of Physics, New York University, New York, NY, USA}
\affil[3]{Max-Planck-Institut f\"ur Astronomie, Heidelberg, Germany}
\affil[4]{Center for Computational Astrophysics, Flatiron Institute, New York, NY, USA}
\affil[5]{Department of Mathematics, Universidad Autónoma de Madrid, Madrid, Spain}
\affil[6]{Mathematical Institute for Data Science, Johns Hopkins University, Baltimore, MD, USA}

\date{}

\begin{document}

\maketitle\thispagestyle{empty}

\paragraph{Abstract:}
Machine learning methods are increasingly being employed as surrogate models in place of computationally expensive and slow numerical integrators for a bevy of applications in the natural sciences.
However, while the laws of physics are relationships between scalars, vectors, and tensors that hold regardless of the frame of reference or chosen coordinate system, surrogate machine learning models are not coordinate-free by default.
We enforce coordinate freedom by using geometric convolutions in three model architectures: a ResNet, a Dilated ResNet, and a UNet.
In numerical experiments emulating 2D compressible Navier-Stokes, we see better accuracy and improved stability compared to baseline surrogate models in almost all cases.
The ease of enforcing coordinate freedom without making major changes to the model architecture provides an exciting recipe for any CNN-based method applied to an appropriate class of problems.

\section{Introduction}

Contemporary natural science features many data sets that are images, lattices, or grids of geometric objects.
These might be observations of intensities (scalars), velocities (vectors), magnetic fields (pseudovectors), or polarizations (2-tensors) on a surface or in a volume.
Any grid of vectors or tensors can be seen as a generalization of the concept of an image in which the intensity in each pixel is replaced with a geometric object --- scalar, vector, tensor, or their pseudo counterparts.
These objects are \emph{geometric} in the sense that they are defined in terms of their transformation properties under geometric operators such as rotation, translation, and reflection.
Likewise, a grid of these objects is also geometric, so we will refer to them as \emph{geometric images}.


There are many questions that we might like to answer about a data set of geometric images.
The images could be the initial conditions of a simulation discretized to a regular grid; see \figref{fig:examples} for some examples.
A critical problem in astronomy, climate science, and many other fields involves modeling the evolution of velocity, pressure, and density fields according to the Navier-Stokes equations.
Traditional numerical solvers are accurate and are considered a robust standard for solving the Navier-Stokes equations, but they can be computationally expensive for systems that require a high resolution. 
Creating surrogate models with machine learning methods has shown promise as an alternative. 
Once trained on the desired spatial and temporal scales, these surrogate models can generate an approximate solution from an initial condition much faster than a traditional solver. 
However, long-term stability in surrogate models remains a concern.

\begin{hoggfigure}
  \begin{center}
    \begin{minipage}[b]{2.5in}\noindent%
      \includegraphics[width=\textwidth]{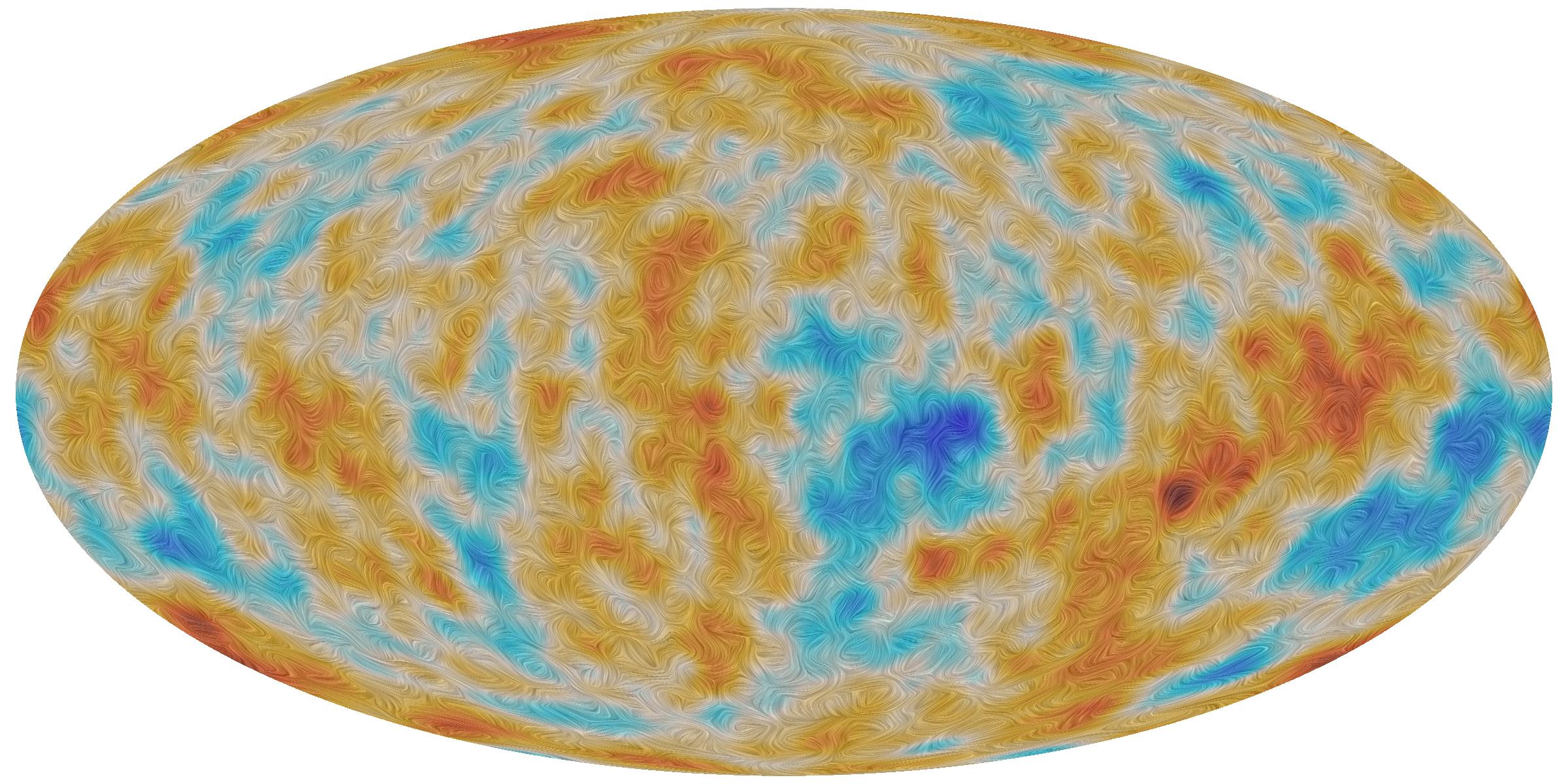}\\
      \textsl{(a)} temperature and polarization
    \end{minipage}
    \begin{minipage}[b]{2in}\noindent%
      \includegraphics[width=\textwidth]{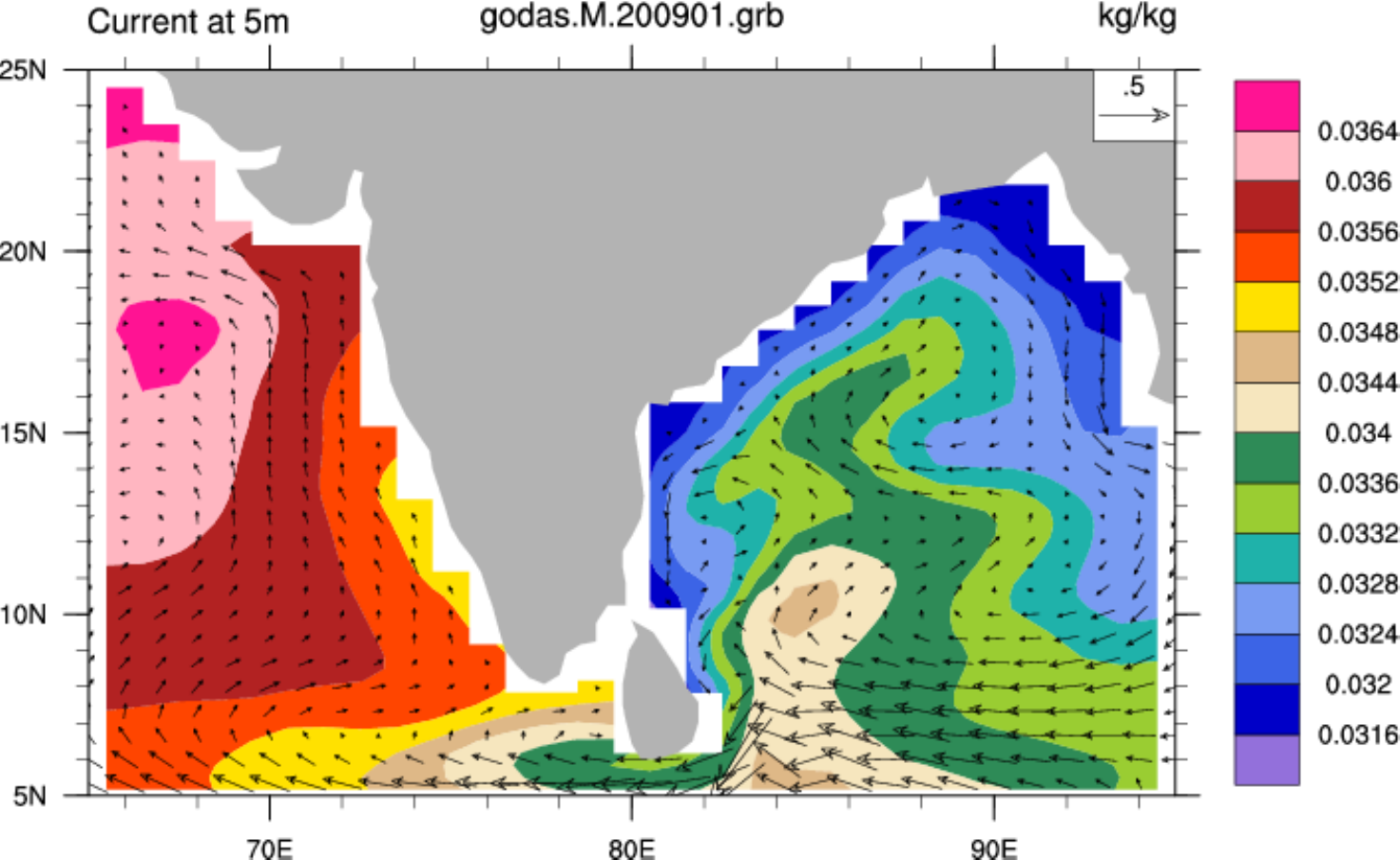}\\
      \textsl{(b)} salinity and current
    \end{minipage}\\[2ex]
    \begin{minipage}[b]{2.5in}\noindent%
      \includegraphics[width=\textwidth]{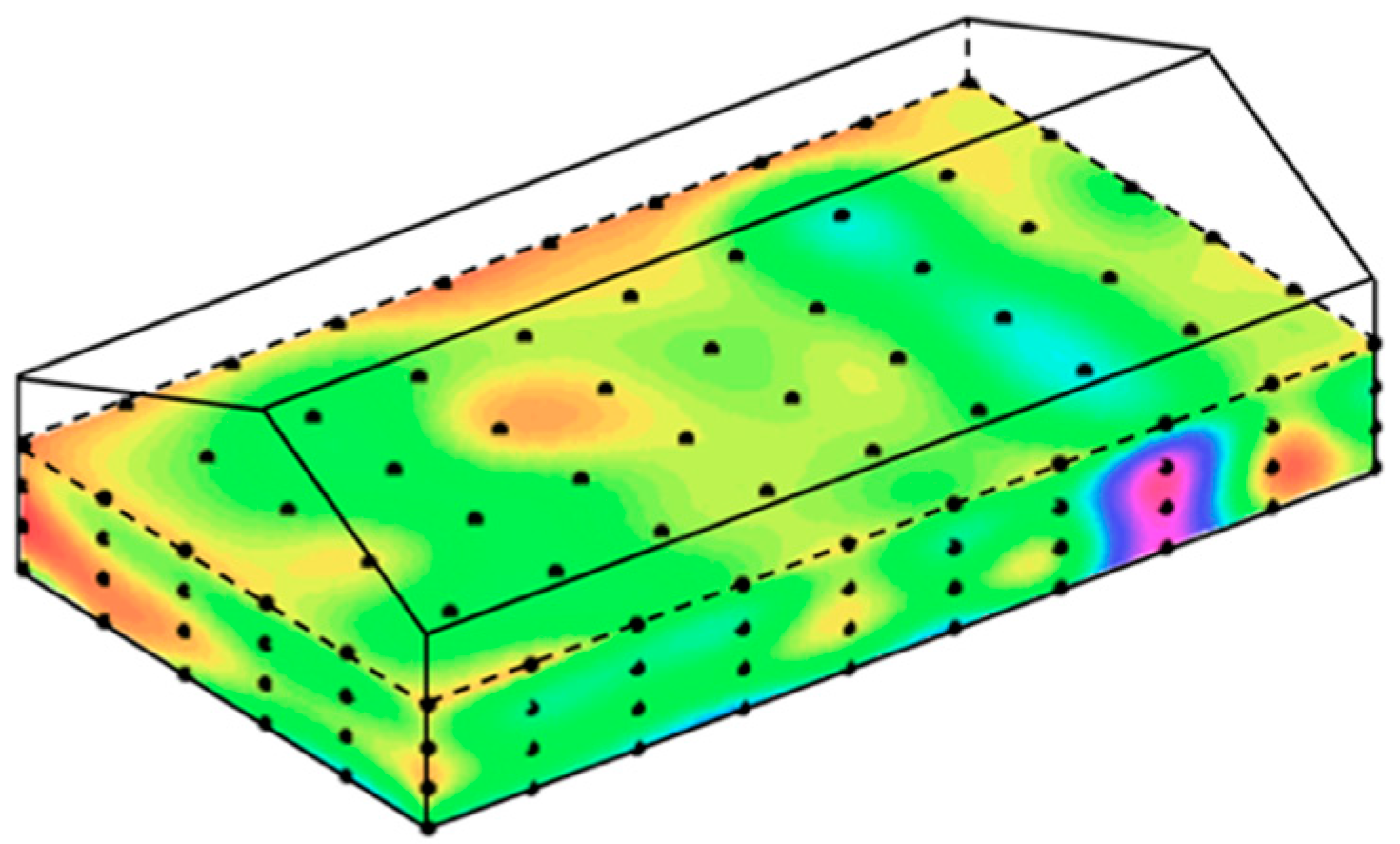}\\
      \textsl{(c)} temperature
    \end{minipage}
    \begin{minipage}[b]{1.8in}\noindent%
      \includegraphics[width=\textwidth]{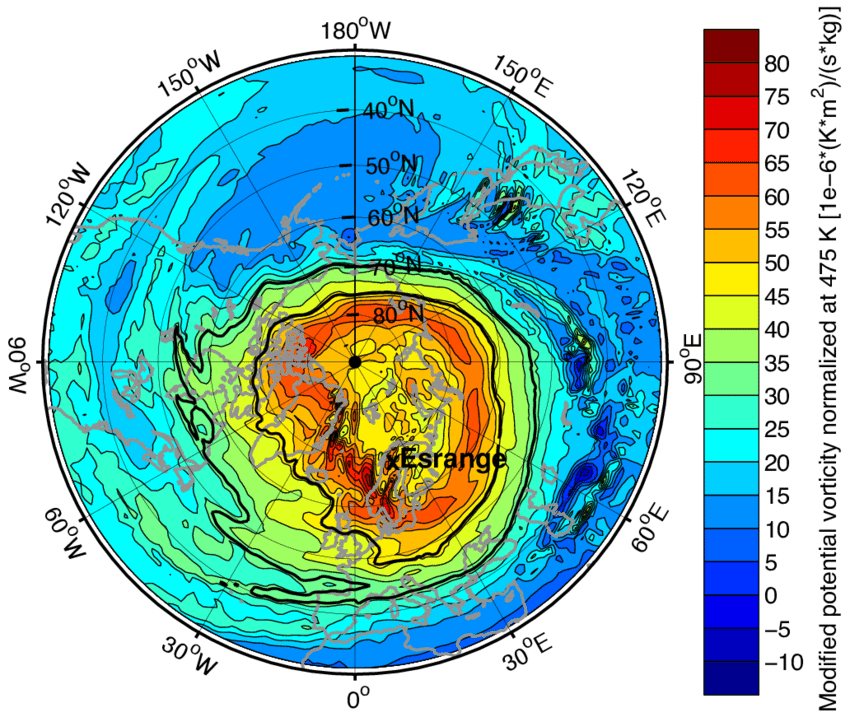}\\
      \textsl{(d)} vorticity
    \end{minipage}\\[2ex]
  \end{center}
  \caption{
  Examples of geometric images in the natural sciences.
  \textsl{(a)}~A visualization of a temperature map and a polarization map from the ESA \textsl{Planck} Mission \cite{planck2015}.
  The color map shows a temperature field (a scalar or \tensor{0}{+}) on the sphere, and the whiskers show the principal eigenvector direction of a \tensor{2}{+} field in two dimensions.
  \textsl{(b)}~Two-dimensional maps of ocean current (arrows; a vector or \tensor{1}{+} field) and ocean salinity (color; a scalar or \tensor{0}{+} field) \cite{climatedataguide}.
  \textsl{(c)}~A three-dimensional map of temperature (a scalar or \tensor{0}{+} field) based on sensors distributed throughout the volume of a granary \cite{granary}.
  \textsl{(d)}~A two-dimensional map of potential vorticity (a pseudoscalar or \tensor{0}{-} field) in the Earth's atmosphere, measured for the purposes of predicting storms \cite{potentialvorticity}.
  \label{fig:examples}}
\end{hoggfigure}

One potential culprit for unstable rollouts is that machine learning models are not coordinate-free by default; they operate on the \emph{components} of the vectors rather than the vectors themselves. 
In typical contexts, the input channels of a convolutional neural network (CNN) are the red, green, and blue channels of a color image; these are then combined arbitrarily in the layers of the CNN.
The naive, flawed approach to applying CNN methods to geometric images is to treat the components of the vector or tensor as independent channels, ignoring how these objects behave under geometric operations.

The fundamental observation inspiring this work is that when an arbitrary function is applied to the components of vectors and tensors, the geometric structure of these objects is destroyed.
There are strict rules, dating back to the early days of differential geometry \cite{ricci}, about how geometric objects can be combined to produce new geometric objects, consistent with coordinate freedom and transformation rules.
These rules constitute a theme of \cite{mcp}, where they are combined into a \emph{geometric principle}.
With the tools of equivariant machine learning, we can make better and more efficient models by incorporating the rules of coordinate freedom.

The concept of equivariance is simple. 
Given a function $f:X\to Y$ and a group $G$ with an action on both $X$ and $Y$, we say $f$ is \textit{equivariant} with respect to $G$ if for all $x \in X$ and $g \in G$ we have $f(g \cdot x) = g \cdot f(x)$.
For equivariant machine learning, we learn a function $f$ over a class of equivariant functions with respect to a relevant group.
Ideally, we would like our group to express all possible coordinate transformations, but this is a very challenging task \cite{villar2024towards}, so in practice, we will consider rotations, reflections, and translations.

The symmetries that these rules suggest are continuous symmetries.
But of course images are usually---and for our purposes---discrete grids of values.
This suggests that instead of the continuous symmetries respected by the tensor objects in the image pixels, there will be discrete symmetries for each geometric image taken as a whole.
We will define these discrete symmetry groups and use them to define a useful kind of group equivariance for functions of geometric images.
When we enforce this equivariance, the convolution filters that appear look very much like the differential operators that appear in discretizations of vector calculus.

The numerical experiments in this work focus on modeling the Navier-Stokes equations which involve scalar fields and vector fields.
However, the model we develop, the \textsl{GeometricImageNet}, can be immediately applied to geometric images of any tensor order or parity.

\paragraph{Our contribution:}
The rest of the paper is organized in the following manner. 
Section \ref{sec:related} discusses related work.
Section \ref{sec:geometric_objects_and_images} defines geometric objects, geometric images, and the operations on each. 
Section \ref{sec:equivariant} discusses the equivariant functions of geometric images with some important results built on the basis of \cite{kondor2018convolution} and \cite{cohen2016group}. 
Sections \ref{sec:architectures} and \ref{sec:experiments} describe how to build a \textsl{GeometricImageNet} and present numerical experiments on compressible Navier-Stokes simulations.
The proofs have been sequestered to the \hyperref[appendix:proofs]{Appendix} along with a larger exploration of related work.

\section{Related work}\label{sec:related}

The difficulty of modeling Navier-Stokes and other PDEs have made the surrogate neural network approach popular in recent years. 
The CNN approach without regard to coordinate freedom is common \cite{stachenfeld2022learned, boltonzanna2019, boltonzanna2020, gupta2022pdearena}, and can be successful with sufficient data.
Some approaches like the Fourier Neural Operator \cite{li2021fourierneuraloperatorparametric} are resolution invariant but not rotationally equivariant.
Other methods have tried to incorporate the physical laws back into ML models under the broad category of physics informed machine learning \cite{karniadakis2021physics, pfortner2022physics}.

Equivariant machine learning is one approach to incorporating physical laws in learned methods by explicitly enforcing the appropriate symmetry in the architecture of the network. 
When we expect our target function to be equivariant to that group, this strategy improves the model's generalization and accuracy (see for instance \cite{elesedy2021provably,wang2021incorporating, huang2024approximately, petrache2023approximation, tahmasebi2023exact}) and is a powerful remedy for data scarcity (see \cite{wang2022data}).
Equivariant networks, in certain cases, can approximate any continuous equivariant function (see \cite{yarotsky2018universal, dym2020universality, bokman2022zz,kumagai2020universal}).

Equivariant models have been built for many different symmetry groups, such as translations \cite{lecun1989backpropagation}, gauge symmetries \cite{cohen2019gauge}, permutations \cite{maron2019invariant}, rotations/reflections \cite{cohen2016group, cohen2016steerable, wang2021incorporating, weiler2021general} or multiple symmetries \cite{thomas2018tensor, dym2020universality}.
There are many approaches to building equivariant models, such as using invariant theory \cite{2022blum-smith-villar-equivariant}, group convolutions \cite{cohen2016group}, canonicalization \cite{kaba2023equivariance}, or irreducible representations \cite{cohen2016steerable, jenner2021steerable, weiler2021general}.
Closest to our paper in both methods and applications are \cite{wang2021incorporating} and \cite{brandstetter2023clifford}, but they implement the symmetries with irreducible representations and Clifford algebras respectively.

Each equivariant method has some challenges.
Group convolutions require convolving over the group elements in addition to the spatial dimensions, which can be expensive for larger groups.
Irreducible representations are often calculated for continuous groups and require sampling to generate discrete (approximately) equivariant filters.
Also, decomposing higher-order tensors into irreducible representations and reconstructing them at the end is a nontrivial task.
The Clifford algebras can handle vectors and pseudovectors naturally, but they cannot handle all higher-order tensors because they are a quotient group of tensor algebra \cite[Ch. 14, Theorem 4.1]{lang02}.
In this work, we use geometric convolutions which will be naturally discrete, exactly equivariant, and able to handle any tensor order or parity.

See Appendix \ref{appendix_sec:related} for a more in-depth description of the mathematical details of the related work.

\section{Geometric Objects and Geometric Images}\label{sec:geometric_objects_and_images}

We define the geometric objects and geometric images that we use to generalize classical images in scientific contexts in \sectionname{} \ref{sec:geometric} and \sectionname{} \ref{sec:convolution}.
The main point is that the channels of geometric images, the components of vectors and tensors, are not independent.
There is a set of allowed operations on geometric objects that respect the structure and coordinate freedom of these objects.

\subsection{Geometric objects}\label{sec:geometric}

We start by fixing $d$, the dimension of the space, which will typically be 2 or 3.
The coordinate transformations will be given by the orthogonal group $O(d)$, the space of isometries of $\mathbb R^d$ that fix the origin. 
The geometric principle from classical physics \cite{mcp} states that geometric objects should be coordinate-free scalars, vectors, and tensors, or their negative-parity pseudo counterparts. 
By coordinate-free we mean that if $F$ is a function with geometric inputs, outputs, and parameters, then  $F(g\cdot v) = g\cdot F(v)$ for all objcts $v$ and all $g \in O(d)$. This is the mathematical concept of equivariance which we will explore further in Section \ref{sec:equivariant}.
This requires that the definitions of the geometric objects are inseparable from how $O(d)$ acts on them.
\begin{definition}[(pseudo-)scalars]\label{def:group_action_on_scalar}
    Let $s \in \mathbb{R}$ have an assigned parity $p \in \qty{-1,+1}$.
    Let $g \in O(d)$ and let $M(g)$ be the standard $d \times d$ matrix representation of $g$, i.e. $M(g^{-1}) = M(g)^{-1} = M(g)^\top$. 
    Then the action of $g$ on $s$, denoted $g \cdot s$, is defined as
    \begin{equation}
        g \vdot s = \det(M(g))^{\frac{1-p}{2}} \,s ~.
    \end{equation}
    When $p=+1$, $s$ is a \textit{scalar}, and $\det(M(g))^{\frac{1-p}{2}} = 1$ so the action is just the identity. When $p=-1$, $s$ is a \textit{pseudoscalar}, so $\det(M(g))^{\frac{1-p}{2}} = \det(M(g)) = \pm 1$ and there is a sign flip if $g$ involves an odd number of reflections.
\end{definition}

\begin{definition}[(pseudo-)vectors]\label{def:group_action_on_vector}
    Let $v \in \mathbb{R}^d$ be a \textit{vector} and let $v$ have parity $p\in\qty{-1,+1}$. 
    Let $g \in O(d)$ and let $M(g)$ be the standard matrix representation of $g$. 
    Then the action of $g$ on $v$, denoted $g \cdot v$, is defined as
    \begin{equation}\label{eq:group_action_on_vector}
        g \vdot v = \det(M(g))^{\frac{1-p}{2}} M(g) \, v ~,
    \end{equation}
    where parity $p$ has the same effect as on the scalars.
\end{definition}
We can now construct higher order tensors using the tensor (outer) product.

\begin{definition}[\tensors{k}{p}]\label{def:tensor}
    The space $\mathbb{R}^d$ equipped with the action $O(d)$ defined by \eqref{eq:group_action_on_vector} is the space of \textit{\tensors{1}{p}}. If we have $k$ \tensors{1}{p_i} denoted $v_i$, then $T:= v_1 \otimes \ldots \otimes v_k$ is a \textit{rank-1 \tensor{k}{p}}, where $p=\prod_{i=1}^k p_i$ and the action of $O(d)$ is defined as 
    \begin{equation}\label{eq:group_action_on_tensor}
        g \cdot \qty(v_1 \otimes \ldots \otimes v_k) = (g \cdot v_1) \otimes \ldots \otimes (g \cdot v_k) \,.
    \end{equation}
    Thus a tensor $T$ is an element of a vector space $(\mathbb{R}^d)^{\otimes k}$, which we denote $\mathcal{T}_{d,k,p}$.
    To get higher rank tensors, we can add tensors of the same order $k$ and parity $p$, and the action of $O(d)$ extends linearly. 
\end{definition}

Note that the parity $p$ is not an intrinsic quality of the components of a tensor.
For example, a vector and a pseudovector could be equal for a certain choice of coordinates, but they would behave differently under some coordinate transformations.
Also note the distinction between the \emph{order} $k$ of the \tensor{k}{p}, and the rank of the tensor. We could have a \tensor{2}{p} of rank 1, like those we use in Definition~\ref{def:tensor}.
We refer to the components of tensors with Einstein summation notation.

\begin{definition}[Einstein summation notation]\label{def:einstein_notation}
In \textit{Einstein summation notation}, the components of tensors are referred to by subscripts, e.g. $[a]_{ij}$ for the $i^{th},j^{th}$ component of \tensor{2}{p} $a$ where $i$ and $j$ are in the range $1,\ldots,d$.
In this paper, we assume that our tensor images have a Riemmannian metric of the identity matrix, so we do not need to distinguish between covariant and contravariant indices.
A subscript index may appear exactly once in a term, in which case we are taking the outer product, or exactly twice, in which case we are summing over (contracting) that index.
\end{definition}

This notation can be used to express a lot of familiar operations. 
For example, the dot product of vectors $a,b$ is written as $[a]_i [b]_i$.
The product of two \tensor{2}{p}s (represented as two $d\times d$ matrices $A$ and $B$) is written as
\begin{equation}
    [A\, B]_{i,j} = [A]_{i,k}\,[B]_{k,j} := \sum_{k=1}^d [A]_{i,k}\,[B]_{k,j}
\end{equation}
where the sum from 1 to $d$ on repeated index $k$ is implicit in the middle expression.
In summation notation, the group action of \eqref{eq:group_action_on_tensor} on \tensor{k}{p} $b$ is explicitly written
\begin{align}\label{def:tensor_rotation}
    [g\cdot b]_{i_1,\ldots, i_k} = \det(M(g))^{\frac{1-p}{2}}\,[b]_{j_1,\ldots,j_k}\,[M(g)]_{i_1,j_1}\cdots[M(g)]_{i_k,j_k}
\end{align} for all $g\in O(d)$.
For example, a \tensor{2}{+} has the transformation property $[g\cdot b]_{i,j} = [b]_{k,\ell}\,[M(g)]_{i,k}\,[M(g)]_{j,\ell}$, which, in normal matrix notation, is written as $g\cdot b = M(g)\,b\,M(g)^\top$. 
To make operations on general \tensor{k}{p} more concise, we adopt the following two defintions.

\begin{definition}[tensor product]\label{def:tensor_product}
    Let $a$ be a \tensor{k}{p} and let $b$ be a \tensor{k'}{p'}. Then the \textit{tensor product of a and b}, denoted $a \otimes b$, is the \tensor{\qty(k+k')}{p\,p'} whose $i_1,\ldots, i_{k+k'}$ components are defined as
    \begin{equation}
        [a \otimes b]_{i_1, \ldots, i_{k+k'}} = [a]_{i_1, \ldots, i_k} \, [b]_{i_{k+1}, \ldots, i_{k+k'}}
    \end{equation}
\end{definition}

\begin{definition}[$k$-contraction]\label{def:contraction}
    Let $a$ be a \tensor{(2k+k')}{p}, then the \textit{k-contraction} $\contract{k}{a}$ is a \tensor{k'}{p} defined as:
    \begin{equation}\label{eq:contraction}
        [\contract{k}{a}]_{j_1, \ldots, j_{k'}} = [a]_{i_1, \ldots, i_k, i_1, \ldots, i_k, j_1 \ldots, j_{k'}}
    \end{equation}
    In other words, we are contracting over indices $(1,k)$ to $(k+1,2k)$.
\end{definition}

It is helpful to think of the contraction as the generalization of the trace to higher order tensors, where we are summing over $k$ pairs of axes.
For a \tensor{2}{p} $a$, the tensor contraction $\contract{1}{a}$ is exactly the trace, a \tensor{0}{p}.
If $a$ is a \tensor{5}{p}, then the contraction $\contract{2}{a}$ is the \tensor{1}{p} given by:
\begin{equation}
    [\contract{2}{a}]_{j} = [a]_{i,\ell,i,\ell,j} = \sum_{i=1}^d \sum_{\ell=1}^d [a]_{i,\ell,i,\ell,j}
\end{equation}

We use the $k$-contraction to define a norm for tensors, which is equivalent to the $\ell_2$ norm on the vectorized tensor or the Frobenius norm for matrices extended to tensors.

\begin{definition}[$\ell_2$ tensor norm]\label{def:frobenius_norm}
    Let $a$ be a \tensor{k}{p}. 
    Then the \textit{$\ell_2$ tensor norm} $\norm{\cdot}_2: \mathcal{T}_{d,k,p} \to \mathcal{T}_{d,0,+}$ is defined as:
    \begin{equation}\label{eq:frobenius_norm}
        \norm{a}_2 = \sqrt{\contract{k}{a \otimes a}}
    \end{equation}
\end{definition}






\subsection{Geometric images and operations}\label{sec:convolution}

We will start by considering square (or cubic or hyper-cubic) images on a $d$-torus.
We work on a $d$-torus to simplify the mathematical results; all the definitions and operations will be applicable with minor adjustments to rectangular, non-toroidal arrays as well. 
We consider an image $A$ with $N$ equally spaced pixels in each dimension for $N^d$ pixels total.
Each pixel contains a \tensor{k}{p} where $k$ and $p$ are the same for each pixel.
We define the geometric images as follows.

\begin{definition}[geometric image] \label{def:geometric_image}
A \textit{geometric image} is a function $A:[N]^d \to \mathcal T_{d,k,p}$, where $[N]=\{0,1,\ldots, N-1\}$. The set of geometric images is denoted $\mathcal{A}_{N,d,k,p}$. We will also consider \tensor{k}{p} images on the $d$-torus, where $[N]^d$ is given the algebraic structure of $(\mathbb Z / N\mathbb Z)^d$. The pixel index of a geometric image, often $\bar\imath$, is naturally a \tensor{1}{+}.
\end{definition}

Just as the space of \tensors{k}{p} is a vector space, the space of geometric images is also a vector space. 
Thus they include vector addition and scalar multiplication. 
Additionally, for each tensor operation defined in Section \ref{sec:geometric}, we can define an analogous operation on geometric images that is performed pixel-wise.

We now turn to the first major contribution of this paper, the generalization of convolution to take geometric images as inputs and return geometric images as outputs.
The idea is that a geometric image of \tensor{k}{p}s is convolved with a geometric filter of \tensors{k'}{p'} to produce a geometric image that contains \tensor{(k+k')}{p\,p'}s, where each pixel is a sum of outer products. 
These \tensor{(k+k')}{p\,p'}s can then be contracted down to lower-order tensors using contractions (Definition~\ref{def:contraction}). 
Note that the sidelength $M$ of the geometric filter can be any positive odd number, but typically it will be much smaller than the sidelength $N$ of the geometric image.

\begin{definition}[geometric convolution]\label{def:convolution}
Given $A \in \mathcal{A}_{N,d,k,p}$ and $C \in \mathcal{A}_{M,d,k',p'}$ with $M=2m+1$ for some positive integer $m$, the \textit{geometric convolution} $A\ast C$ is a \tensor{(k+k')}{p\,p'} image such that
\begin{equation}\label{eq:convolution}
    (A\ast C)(\bar\imath) = \sum_{\bar a\in[-m, m]^d} A(\bar\imath - \bar a)\otimes C(\bar a + \bar m) ~,
\end{equation}
where $\bar\imath - \bar a$ is the translation of $\bar\imath$ by $\bar a$ on the $d$-torus pixel grid $(\mathbb Z / N\mathbb Z)^d$ and $\bar m$ is the vector of all $m$.
\end{definition}
This definition is on the torus to achieve exact translation equivariance, but in practice we can use zero padding or any other form of padding as the situation requires. 
Additionally, geometric convolution can be adapted to use longer strides, filter dilation, transposed convolution, or other convolution variations common in the literature.
See Figure \ref{fig:convolution_example_and_architecture}(a) for examples with a scalar and vector filter.
We can define max pooling using the $\ell_2$ norm of a tensor as follows:

\begin{definition}[$\text{max\,pool}_b$]\label{def:max_pool}
    Let $b$ be a positive integer and let $A \in \mathcal{A}_{N,d,k,p}$, where $b$ divides $N$. Then the function $\text{max\,pool}_b: \mathcal{A}_{N,d,k,p} \to \mathcal{A}_{N/b,d,k,p}$ is defined for each pixel index $\bar\imath \in [0,(N/b)-1]^d$:
    \begin{equation}\label{eq:max_pool}
        \text{max\,pool}_b(A)(\bar\imath) = A\qty(b\,\bar\imath + \argmax_{\bar a \in [0,b-1]^d}\norm{A(b\,\bar\imath + \bar a)}_2) 
    \end{equation}
\end{definition}

The convolution, contraction, index-permutation, and pooling operators above effectively span a large class of linear functions from geometric images to geometric images. 

\section{Functions of geometric images and equivariance}\label{sec:equivariant}

We start by defining equivariance and invariance for a general group $G$, and then we will describe the groups of interest and several theoretical results.

\begin{definition}[Equivariance of a geometric image function]
    Given a function on geometric images $f:\mathcal{A}_{N,d,k,p} \to \mathcal{A}_{N,d,k',p'}$, and a group $G$ equipped with actions on $\mathcal{A}_{N,d,k,p}$ and $\mathcal{A}_{N,d,k',p'}$, we say that $f$ is \textit{equivariant} to $G$ if for all $g \in G$ and $A \in \mathcal{A}_{N,d,k,p}$ we have:
    \begin{equation}
        f(g \cdot A) = g \cdot f(A)
    \end{equation}
    Likewise, $f$ is \textit{invariant} to $G$ if
    \begin{equation}
        f(g \cdot A) = f(A) ~.
    \end{equation}
    We also say a geometric image is $G$-isotropic if $g \cdot A = A$ for all $g \in G$.
\end{definition}

We first consider discrete translations on the $d$-torus pixel grid.
If $A$ is a \tensor{k}{p} image and $\tau \in (\mathbb{Z}/N\mathbb{Z})^d$ then the action $L_\tau A$ produces the \tensor{k}{p} image $(L_\tau A)(\bar\imath)=A(\bar\imath -\tau)$ where $\bar\imath$ is a pixel index and $\bar\imath - \tau$ is the translation of $\bar\imath$ by $\tau$ on the $d$-torus pixel grid. 
The fundamental property of convolution is that it is translation equivariant, and that every translation equivariant linear function can be expressed as a convolution with a fixed filter, as long as the filter can be set to be as large as the image \cite{kondor2018convolution}. 
The same property holds for geometric images.

\begin{proposition}\label{prop:convolution_iff_translation_invariant}
A function $f:\mathcal A_{N,d,k,p}\to \mathcal{A}_{N,d,k',p'}$ is a translation equivariant linear function if and only if it can be written as $\contract{k}{A \ast C}$ for some geometric filter $C \in \mathcal{A}_{M,d,k+k',p\,p'}$. When $N$ is odd, $M=N$, otherwise $M=N+1$.
\end{proposition}

See Appendix \ref{proof:convolution_iff_translation_invariant} for the proof.
In addition to translation symmetries, we want to consider other natural symmetries occurring in the application domains where vectors and tensors arise. 
Ideally we would like to apply continuous rotations to the images, but the discretized nature of images makes this challenging. 
To obtain exact results on images, we focus on discrete rotations.
For 2D images this is the familiar dihedral group $D_4$ of rotations of 90 degrees and reflections, and in the general-D case it is the hyperoctahedral group $B_d$, the Euclidean symmetries of the $d$-dimensional hypercube.
The notation $B_d$ is standard nomenclature coming from the classification theorem for finite irreducible reflection groups \cite{humphreys1990reflection}.
Because the groups $B_d$ are subgroups of $O(d)$, all determinants of the matrix representations of the group elements are either $+1$ or $-1$, and the matrix representation $M(g^{-1})$ of the inverse $g^{-1}$ of group element $g$ is the transpose of the matrix representation $M(g)$ of group element $g$.


\begin{definition}[Action of $B_d$ on \tensor{k}{p}s]\label{def:action_bd_on_tensor}
    Given a \tensor{k}{p} $b$, the action of $g \in B_d$ on $b$, denoted $g \cdot b$, is the restriction of the action in Definition \ref{eq:group_action_on_tensor} to $B_d$ which is a subgroup of $O(d)$.
\end{definition}

\begin{definition}[Action of $B_d$ on \tensor{k}{p} images]\label{def:action_on_tensor_image}
Given $A \in \mathcal{A}_{N,d,k,p}$ on the $d$-torus and a group element $g\in B_d$, the action $g\cdot A$ produces a \tensor{k}{p} image on the $d$-torus such that
\begin{equation}
    (g\cdot A)(\bar\imath) = g\cdot A({g^{-1}\cdot \bar\imath}) ~.
\end{equation}
Since $\bar \imath$ is a \tensor{1}{+}, the action $g^{-1} \cdot \bar \imath$ is performed by centering $\bar\imath$, applying the operator, then un-centering the pixel index: 
$$
    g^{-1} \cdot \bar \imath = \qty(M(g^{-1})(\bar \imath - \bar m))+\bar m
$$ 
where $\bar m$ is the $d$-length \tensor{1}{+} $\qty[\frac{N-1}{2}, \ldots, \frac{N-1}{2}]^\top$. If the pixel index is already centered, such as $\bar a \in [-m,m]^d$, then we skip the centering and un-centering.
\end{definition}

It might be a bit surprising that the group element $g^{-1}$ appears in the definition of the action of the group on images.
One way to think about it is that the pixels in the transformed image are ``looked up'' or ``read out'' from the pixels in the original untransformed image.
The pixel locations in the original image are found by going back, or inverting the transformation.

\begin{definition}[The group $G_{N,d}$, and its action on \tensor{k}{p} images]\label{def:GdN} $G_{N,d}$ is the group generated by the elements of $B_d$ and the discrete translations on the $N^d$-pixel lattice on the $d$-torus.
\end{definition}

\begin{remark}
    We view the $d$-torus as the quotient of the $d$-hypercube obtained by identifying opposite faces. 
    The torus obtains the structure of a flat (i.e., zero curvature) Riemannian manifold this way. Because the symmetries $B_d$ of the hypercube preserve pairs of opposite faces, they act in a well-defined way on this quotient, so we can also view $B_d$ as a group of isometries of the torus. 
    We choose the common fixed point of the elements of $B_d$ as the origin for the sake of identifying the $N^d$ pixel lattice with the group $T_{N,d} \cong (\mathbb{Z}/N\mathbb{Z})^d$ of discrete translations of this lattice; then the action of $B_d$ on the torus induces an action of $B_d$ on $T_{N,d}$ by automorphisms. 
    The group $G_{N,d}$ is the semidirect product $T_{N,d}\rtimes B_d$ with respect to this action. Thus there is a canonical group homomorphism $G_{N,d} \rightarrow B_d$ with kernel $T_{N,d}$. 
    In concrete terms, every element of $G_{N,d}$ can be written in the form $\tau \circ b$, where $b\in B_d$ and $\tau \in T_{N,d}$. 
    Then the canonical map $G_{N,d}\rightarrow B_d$ sends $\tau \circ b$ to $b$. 
\end{remark}

Now that we have defined the group that we are working with, we can specify how to build convolution functions that are equivariant to $G_{N,d}$. 
The following theorem generalizes the Cohen and Welling paper \cite{cohen2016group} for geometric convolutions.

\begin{theorem}\label{theorem:linear_equiv_characterization}
    A function $f:\mathcal{A}_{N,d,k,p} \to \mathcal{A}_{N,d,k',p'}$ is linear and $G_{N,d}$-equivariant if and only if it can be written as $\contract{k}{A * C}$ for some $B_d$-isotropic $C \in \mathcal{A}_{M,d,k+k',p\,p'}$, where $M=N$ if $N$ is even and $M=N+1$ otherwise.
\end{theorem}

The proof of this theorem is given in Appendix \ref{appendix:proofs}.
Theorem \ref{theorem:linear_equiv_characterization} provides the explicit requirements for linear layers in our equivariant \textsl{GeometricImageNet}. 
All we need are the $B_d$-isotropic \tensor{(k+k')}{p\,p'} filters which are straightforward to find using group averaging.

\begin{figure}
    \begin{centering}
        \begin{minipage}{0.49\linewidth}
            \includegraphics[width=\linewidth]{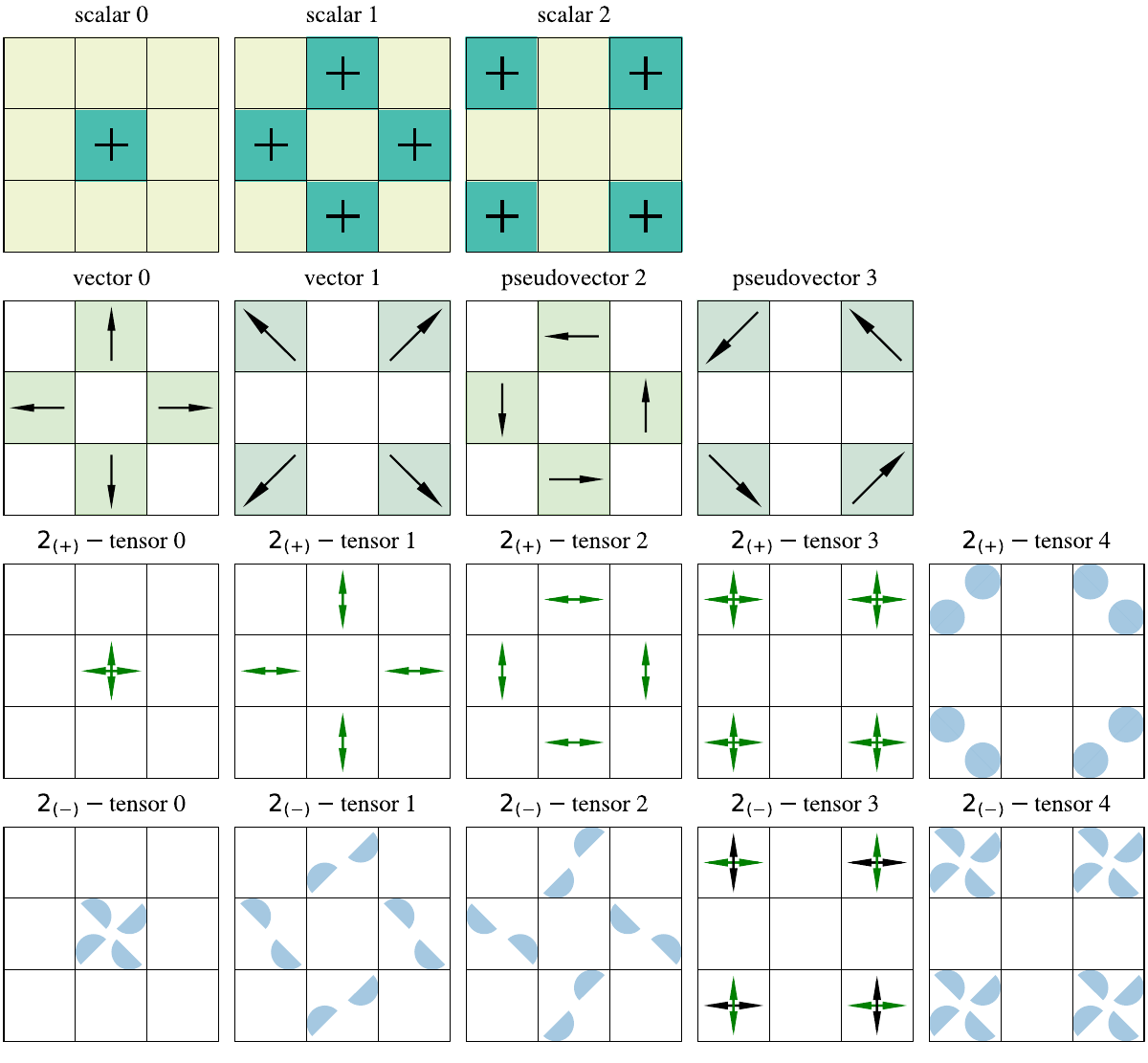}
        \end{minipage}
        \begin{minipage}{0.49\linewidth}
            \includegraphics[width=\linewidth]{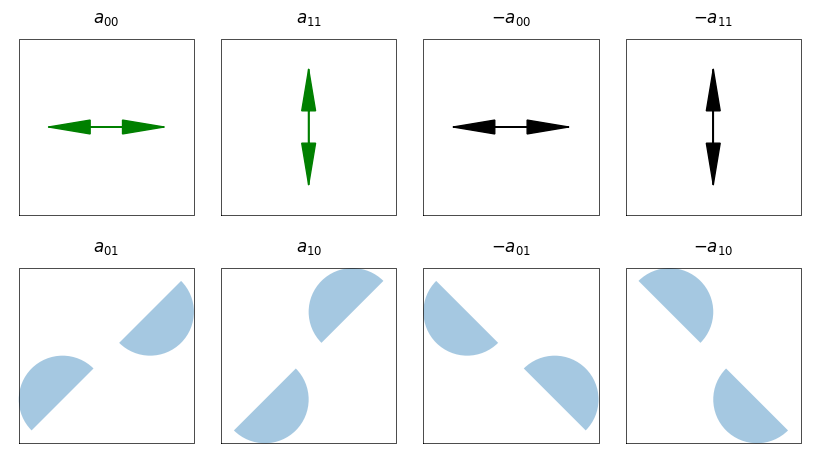}
        \end{minipage}
        \begin{minipage}{0.49\linewidth}
            \centering
            (a)
        \end{minipage}
        \begin{minipage}{0.49\linewidth}
            \centering
            (b)
        \end{minipage}
    \end{centering}
    \caption{
    (a) All the filters for $d=2$, $M=3$, $k \in \qty{0,1,2}$.
    Where there is no symbol in the box the value is zero.
    There are no $B_d$-isotropic pseudoscalar filters at $d=2, M=3$.
    (b) Each signed component in the \tensor{2}{p} has a particular icon, with the positive diagonal elements represented by the green double arrows, the negative diagonal elements represented by the black double arrows, and the off diagonal elements represented by the petals.
    Each element rotates in the obvious way, and \tensors{2}{+} reflect in the obvious way as well.
    However, reflections on negative parity diagonal elements flip the sign (color) of the double arrows and have no effect on the petals other than changing their pixel location.
    }
    \label{fig:filters23}
\end{figure}

 
\section{\textsl{GeometricImageNet} Architectures}\label{sec:architectures}

Per Theorem \ref{theorem:linear_equiv_characterization}, we construct linear $G_{N,d}$-equivariant layers using $B_d$-isotropic filters. 
A complete basis of $B_d$-isotropic \tensor{(k+k')}{p\,p'} filters can be found by group averaging.
First we get the standard basis of $\mathbb{R}^{M^d \times d^{(k+k')}}$ and reshape them into filters $C_i$ with sidelength $M$ and assigned parity $p\,p'$.
Next we apply the group averaging:
\begin{equation}\label{eq:group_averaging}
   \widetilde{C}_i = \frac{1}{\qty|B_d|}\sum_{g \in B_d} g \cdot C_i ~,
\end{equation}
where $|B_d|$ is the number of group elements.
This will likely result in a linearly dependent set of filters, so we perform SVD to reduce to a single set of unique filters.
The filters are then normalized so that non-zero tensors have unit norm, and the $k=1$ filters are also reoriented such that non-zero divergences were set to be positive, and non-zero curls were set to be counter-clockwise.
See \figurename~\ref{fig:filters23} for the $B_d$-isotropic convolutional filters in $d=2$ dimensions for filters of sidelength $M=3$. 
Next, we use these $B_d$-isotropic filters to construct linear $G_{N,d}$-equivariant layers.

The linear layers take an input collection of geometric images $\qty{(k_i,p_i)}_{i=1}^{W_{\text{in}}}$ with $c_i$ channels and the desired output tensor orders and parities $\qty{(k_j,p_j)}_{j=1}^{W_{\text{out}}}$ with $c_j$ channels and computes all the convolutions\footnote{The geometric convolution package is implemented in JAX, which in turn uses TensorFlow XLA under the hood. This means that convolution is actually cross-correlation, in line with how the term in used in machine learning papers. For our purposes this results in at most a coordinate transformation in the filters.} and contractions to map between those two sets. 
Following Theorem \ref{theorem:linear_equiv_characterization}, there are $\ell = 1,\ldots,c_j$ functions $\sum_{i=1}^{W_\text{in}}\sum_{z=1}^{c_i} \contract{k_i}{A_{i,z}*C_{\ell,i,z}}$ for each desired output tensor order and parity.
Per the theorem, these convolution filters $C_{\ell,i,z}$ must be $B_d$-isotropic to guarantee that this layer is $G_{N,d}$-equivariant.
Each $B_d$-isotropic filter is a parameterized linear combination of the $B_d$-isotropic basis we found by group averaging.
However, using filters as large as the input image is impractical in most cases, so we use deeper networks of $3 \times 3$ or $5 \times 5$ filters, as is commonly done in CNNs \cite{vgg}.

Nonlinear layers present a challenge because the typical pointwise nonlinear functions such as ReLU or tanh break equivariance when applied to the individual components of a tensor.
Properly building $O(d)$-equivariant nonlinear functions is a challenging and active area of research; for a larger exploration, see \cite{xu2022unified} and references therein.
For this model, we extend the Vector Neuron nonlinearity \cite{deng2021vectorneuronsgeneralframework} for any tensor order and parity, see Appendix \ref{appendix:nonlinear} for a precise definition.
See Figure \ref{fig:convolution_example_and_architecture}(b) for an example of a typical architecture interlacing linear and nonlinear layers.

The final layer types we will use in our model are LayerNorm \cite{ba2016layernormalization} and max pool.
To make an equivariant version of LayerNorm, we follow the strategy of vector whitening used in \cite{brandstetter2023clifford}, based on a similar strategy developed for neural networks with complex values \cite{trabelsi2018deepcomplexnetworks}.
Max pooling layers use the $\ell_2$ tensor norm to determine the max tensor for each channel of each input image.

\begin{hoggfigure}
    \begin{center}
        \begin{minipage}{0.5\linewidth}
            \begin{minipage}{0.3\linewidth}
                \centering
                \includegraphics[width=\linewidth]{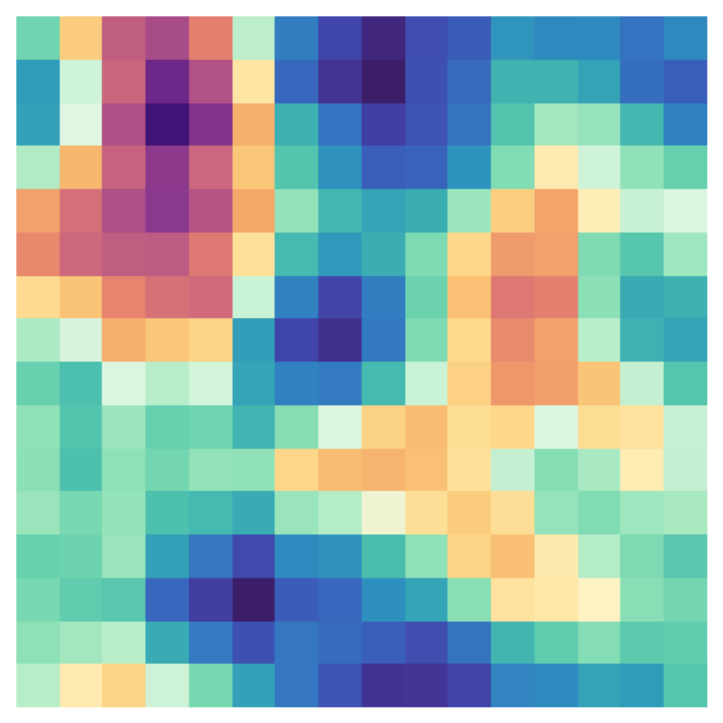}
            \end{minipage}
            \begin{minipage}{0.05\linewidth}
                \centering
                \Large
                *
            \end{minipage}
            \begin{minipage}{0.19\linewidth}
                \centering
                \includegraphics[width=\linewidth]{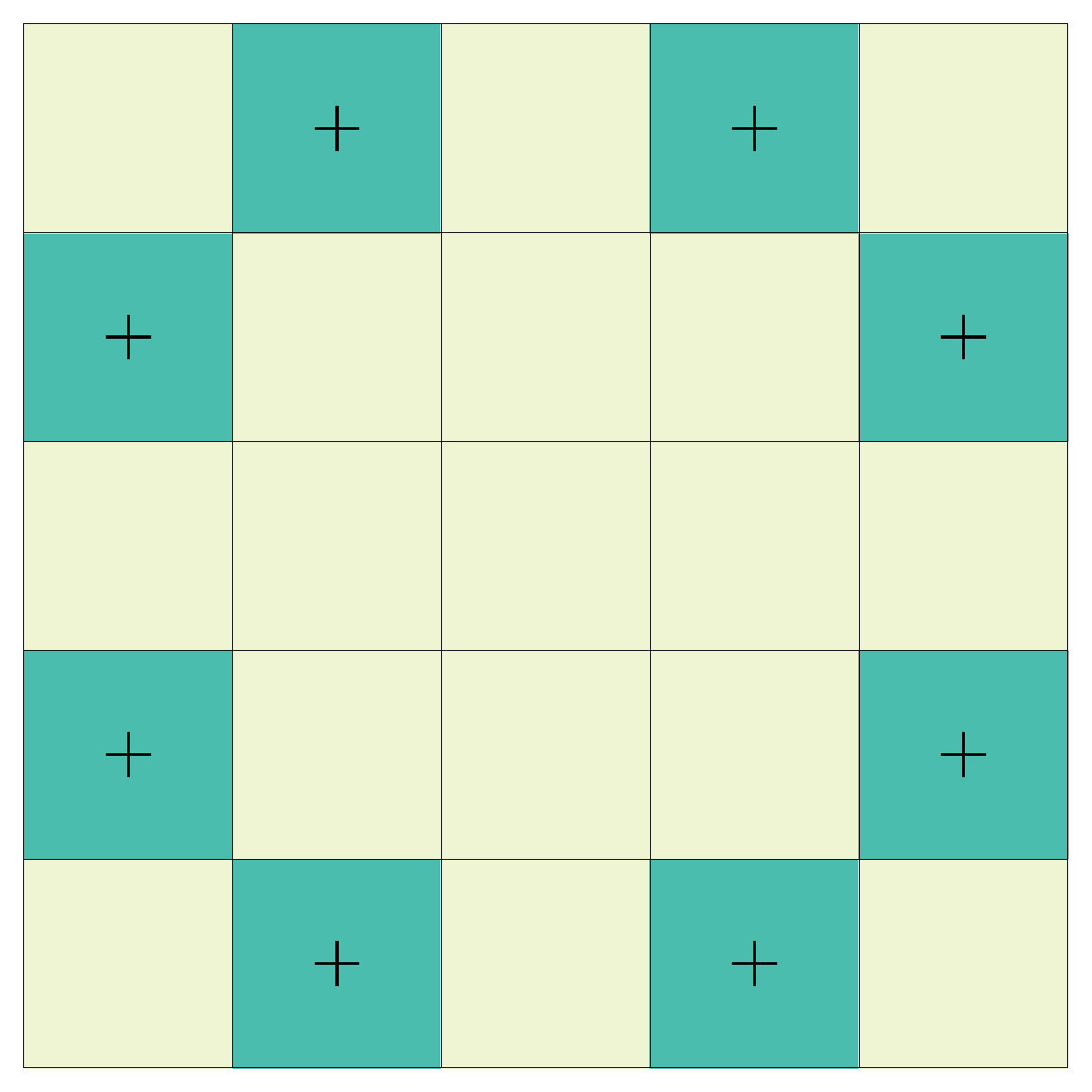}
            \end{minipage}
            \begin{minipage}{0.05\linewidth}
                \centering
                \Large
                =
            \end{minipage}
            \begin{minipage}{0.3\linewidth}
                \centering
                \includegraphics[width=\linewidth]{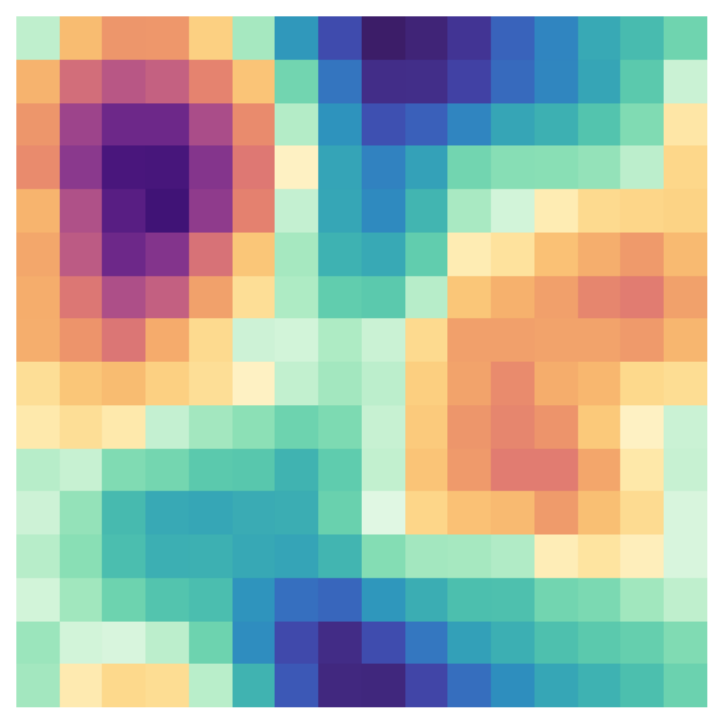}
            \end{minipage}
            
            \begin{minipage}{0.3\linewidth}
                \centering
                \includegraphics[width=\linewidth]{pics/scalar_img.pdf}
            \end{minipage}
            \begin{minipage}{0.05\linewidth}
                \centering
                \Large
                *
            \end{minipage}
            \begin{minipage}{0.19\linewidth}
                \centering
                \includegraphics[width=\linewidth]{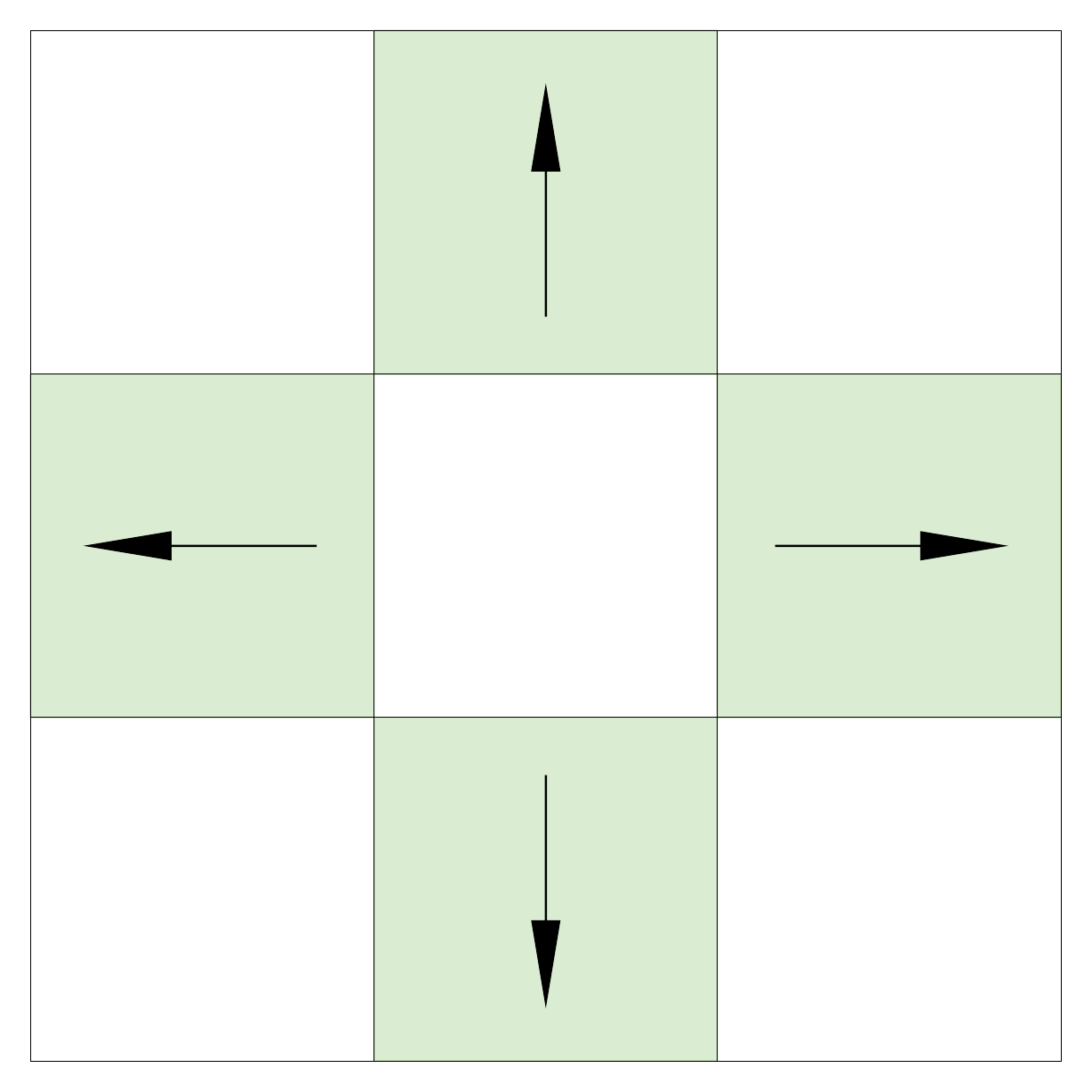}
            \end{minipage}
            \begin{minipage}{0.05\linewidth}
                \centering
                \Large
                =
            \end{minipage}
            \begin{minipage}{0.3\linewidth}
                \centering
                \includegraphics[width=\linewidth]{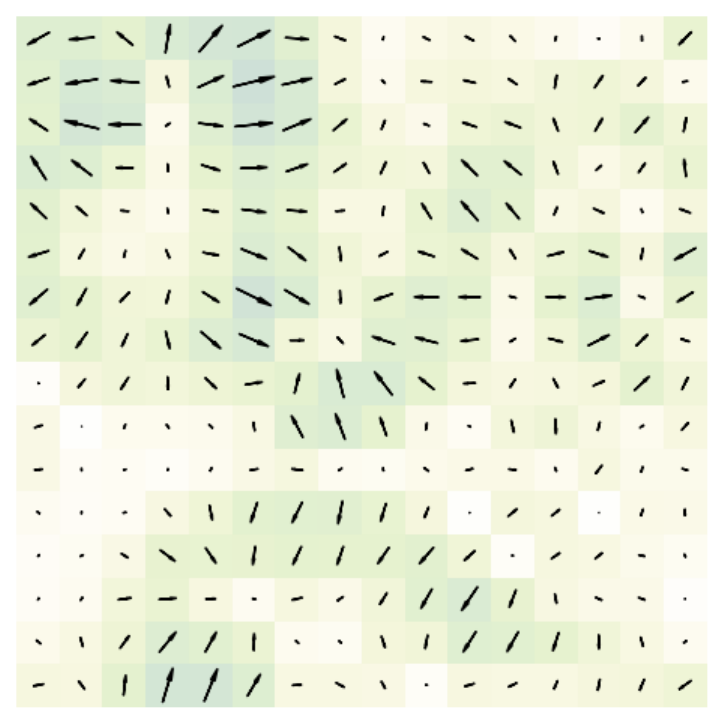}
            \end{minipage}
        \end{minipage}
        \begin{minipage}{0.49\linewidth}
            \centering
            \includegraphics[width=\linewidth]{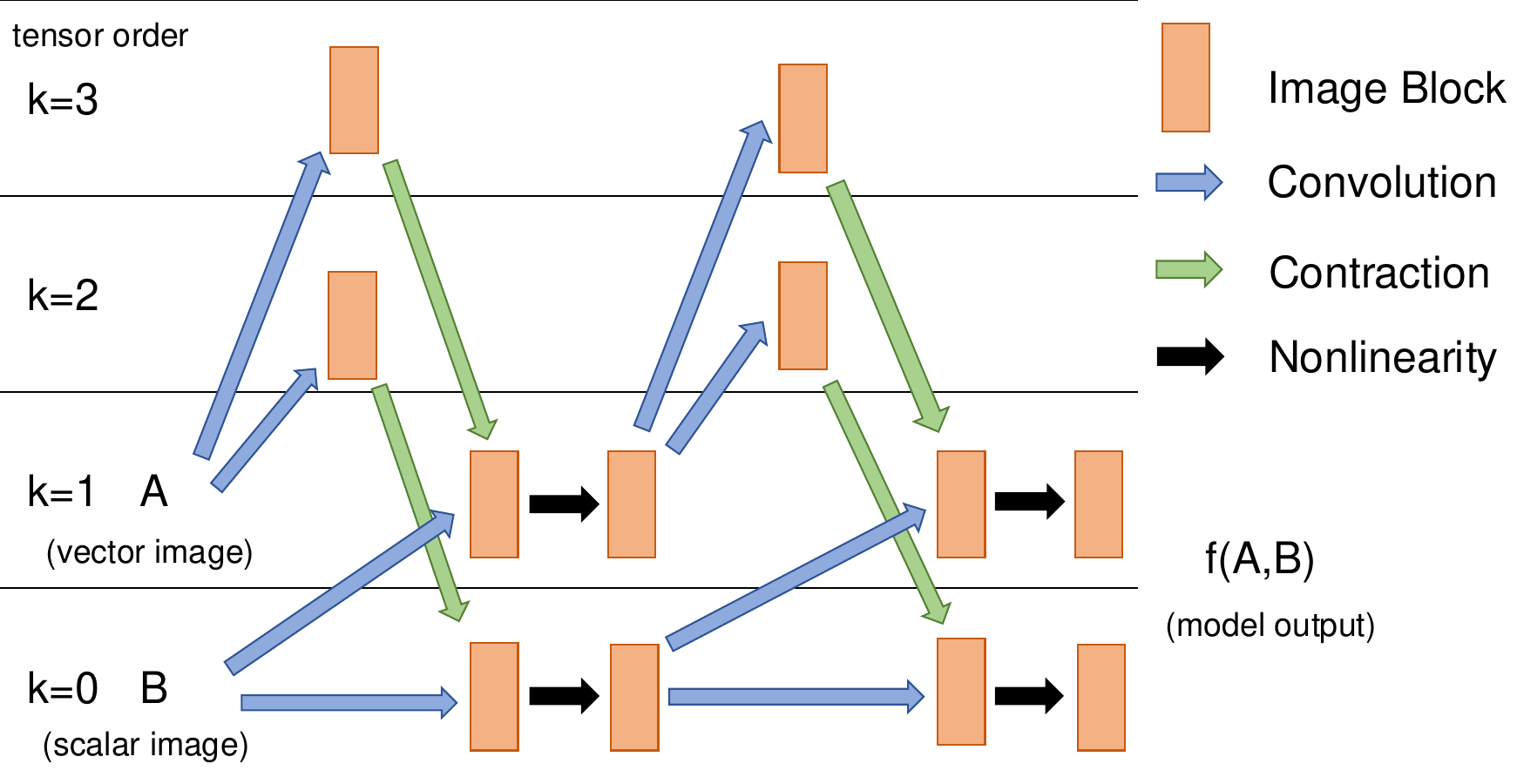}
        \end{minipage}
        \begin{minipage}{0.5\linewidth}
            \centering
            (a)
        \end{minipage}
        \begin{minipage}{0.49\linewidth}
            \centering
            (b)
        \end{minipage}
    \end{center}
    \caption{(a) Convolution of a scalar image with a scalar and vector filter. (b) Example architecture taking a vector image and scalar image as input and output. Linear layers are shown by the blue convolution arrows followed by green contraction arrows. The black arrows represent nonlinearities. The orange blocks represent multiple channels of images at that tensor order.}
    \label{fig:convolution_example_and_architecture}
\end{hoggfigure}


\section{Numerical Experiments}\label{sec:experiments}

We will conduct numerical experiments on 2D compressible Navier-Stokes simulation data from the excellent PDEBench data set \cite{takamoto2024pdebench}.
This data consists of velocity (vector) fields, density (scalar) fields, and pressure (scalar) fields with periodic boundary conditions discretized into $128 \times 128$ images on the torus.
The simulations are saved at $21$ time points which are a subset of the integrator timesteps.
We use $128$ simulation trajectories with random initial conditions as training data and another $128$ trajectories as test data.
We use data generated with two distinct set of parameters: Mach number $M=0.1$, shear viscosity $\eta=0.01$, and bulk viscosity $\zeta=0.01$ and $M=1.0,\eta=0.1,\zeta=0.1$.
The two sets of parameters are used to train entirely different models and tested separately.

The model task is take as input the velocity, density, and pressure fields at a certain time point, and predict what those fields will be at the next time point.
However, since this is too difficult, we actually give the model the four previous time points.
Thus we can turn the $128$ training trajectories into $2,176$ training data points because each trajectory has $17$ overlapping sections of four input steps and one output step.
We train a Dilated ResNet \cite{stachenfeld2022learned}, a ResNet \cite{he2015resnet}, and a UNet \cite{ronneberger2015unet} with and without LayerNorm \cite{ba2016layernormalization} and equivariant versions of each of those models.
We train with the sum of the MSE loss of each field of a single step, but at test time we are also interested in the performance of autoregressively rolling out the model over 15 time steps.
The baseline models and training setup generally follow those described in \cite{gupta2022pdearena}, and additional data, model, and training details are in Appendix \ref{appendix:experiments}.

The numerical results are given in Table \ref{tab:cfd_results}. 
In all cases of the 1-step loss and almost all cases of the 15 step rollout loss, the equivariant models outperform the non-equivariant versions.
In Figure \ref{fig:cfd_results}, we can see with more granularity the test performance for each rollout step.
In the most drastic example, the rollout error for the Dilated ResNet explodes, while the equivariant Dilated ResNet is stable and accurate over all 15 steps.
In \cite{stachenfeld2022learned}, the authors combat this issue by adding a small amount of Gaussian noise during training; we instead achieve stability in a physically-motivated way by enforcing $O(d)$-equivariance.
The equivariance also helps with parameter efficiency; we chose channel depth so that the equivariant model was comparable to the baseline model in parameter count (Table \ref{tab:cfd_results}), however we could also have aimed for the same accuracy to get a large reduction in the number of parameters. 
Code to reproduce all these experiments and build your own GI-Net is available at \url{https://github.com/WilsonGregory/GeometricConvolutions}. The code is built in Python using JAX \cite{jax2018github}.

\begin{hoggfigure}
    \includegraphics[width=\textwidth]{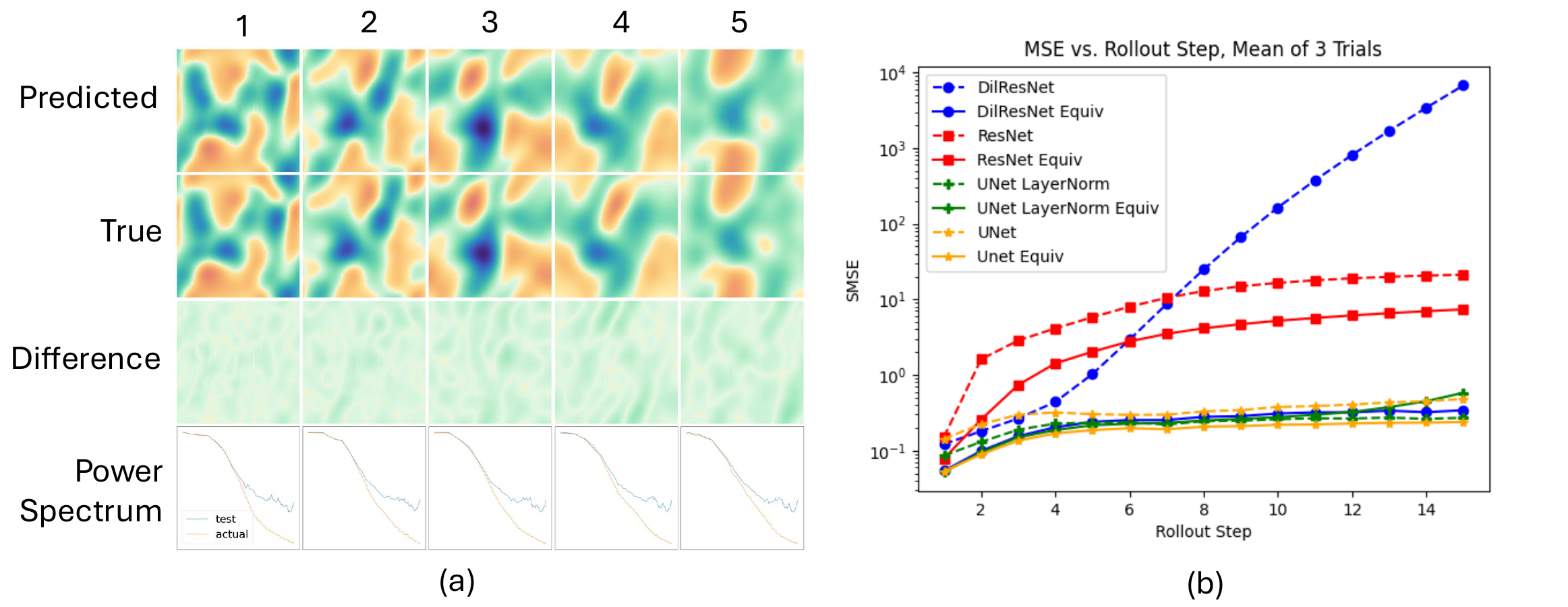}
    \caption{(a) Five steps of M0.1 rollout using the best performing model, the equivariant UNet without LayerNorm. 
    The x-component of the velocity is plotted.
    (b) Comparison of test performance over a 15 step rollout on the M0.1 data set.
    The SMSE is shown for \textit{each} step, rather than a cumulative loss.
    }
    \label{fig:cfd_results}
\end{hoggfigure}

\begin{table}[h]
    \begin{tabular}{c|cc|cc}
        model & M0.1 1-step & M0.1 rollout & M1.0 1-step & M1.0 rollout \\
        \toprule
        DilResNet & 0.040  & 13318.773 $\pm$ 18824.855  & 0.005  & 9.574 $\pm$ 9.608 \\
        DilResNet Equiv & \textbf{0.018 }& \textbf{3.770 $\pm$ 0.090 }& \textbf{0.001 }& \textbf{0.153 $\pm$ 0.023}\\
        \midrule
        ResNet & 0.039  & 175.736 $\pm$ 17.846  & 0.009  & \textbf{0.835 $\pm$ 0.097}\\
        ResNet Equiv & \textbf{0.024 $\pm$ 0.001 }& \textbf{57.508 $\pm$ 9.157 }& \textbf{0.003 }& 2.943 $\pm$ 0.992 \\
        \midrule
        UNet LayerNorm & 0.027  & 3.414 $\pm$ 0.217 & 0.009 $\pm$ 0.001 & 1.067 $\pm$ 0.190 \\
        UNet LayerNorm Equiv & \textbf{0.018 $\pm$ 0.002 }& \textbf{3.971 $\pm$ 1.158 }& \textbf{0.001 }& \textbf{0.136 $\pm$ 0.047}\\
        \midrule
        UNet & 0.047 $\pm$ 0.001 & 5.086 $\pm$ 0.105 & 0.012 $\pm$ 0.002 & 2.074 $\pm$ 0.066 \\
        Unet Equiv & \textbf{0.018 }& \textbf{2.813 $\pm$ 0.257 }& \textbf{0.001 }& \textbf{0.124 $\pm$ 0.018}\\
        \bottomrule
    \end{tabular}
    \caption{ Loss values for each model, averaged over three trials.
    All losses are the sum of the mean squared error losses over the channels: density, pressure, and velocity.
    The rollout loss is the sum of the error over 15 steps. 
    The std $\pm 0.xxx$ is provided if its at least $0.001$.}
    \label{tab:cfd_results}
\end{table}


\section{Discussion}\label{sec:discussion}

This paper presents geometric convolutions which can easily adapt any CNN architecture to be equivariant for images of vectors or tensors.
This makes the model ideal for tackling many problems in the natural sciences in a principled way.
We see in 2D compressible Navier-Stokes simulations that we achieve better accuracy and more stable rollouts than non-equivariant baseline models.

One limitation of this work is that we use discrete symmetries instead of continuous symmetries. 
We expect invariance and equivariance with respect to rotations other than 90 degrees to appear in nature, but the images that we work with are always going to be $d$-cube grids of points.
Thus, we use the group $G_{N,d}$ to avoid interpolating rotated images and working with approximate equivariances.
This simplifies the mathematical results, and we see empirically that we still have the benefits of rotational equivariance. 
However, there are other possible image representations that might create continuous concepts of images. 
For example, if the data is on the surface of a sphere, it could be represented with tensor spherical harmonics, and it could be subject to transformations by a continuous rotation group.

Another limitation of this work is that we do not compare our method to existing state-of-the-art numerical integrator methods.
Surrogate ML models for fluid dynamics simulations have generally suffered from comparisons to weak baselines that overstate the accuracy or efficiency of the surrogate model \cite{mcgreivy_hakim2024}.
In this work, we only claim to improve upon existing vanilla CNN models, and we leave further comparisons to future work.

There are many other future directions that could be explored. 
Further research is required to understand how and why the equivariance helps.
One interesting observation of Figure \ref{fig:cfd_results} (a) is that the power spectrum for the equivariant model output is still quite different from the ground truth at higher frequencies.
It may be that equivariance is advantageous at certain scales and not at others.


\paragraph{Acknowlegements:}
It is a pleasure to thank
Roger Blandford (Stanford),
Drummond Fielding (Flatiron),
Leslie Greengard (Flatiron),
Ningyuan (Teresa) Huang (JHU),
Kate Storey-Fisher (NYU),
and the Astronomical Data Group at the Flatiron Institute for valuable discussions and input.
This project made use of Python~3 \cite{python3}, numpy \cite{numpy}, matplotlib \cite{matplotlib}, and cmastro \cite{cmastro}.
All the code used for making the data and figures in this paper is available at \url{https://github.com/WilsonGregory/GeometricConvolutions}.

\paragraph{Funding:}
WG was supported by an Amazon AI2AI Faculty Research Award.
BBS was supported by ONR N00014-22-1-2126.
MTA was supported by 
H2020-MSCA-RISE-2017, Project 777822, and from
Grant PID2019-105599GB-I00, Ministerio de Ciencia, Innovaci\'on y Universidades, Spain.
SV was partly supported by the NSF–Simons
Research Collaboration on the Mathematical and Scientific Foundations of Deep Learning (MoDL) (NSF DMS
2031985), NSF CISE 2212457, ONR N00014-22-1-2126, NSF CAREER 2339682, and an Amazon AI2AI Faculty Research Award.
\bibliographystyle{plain}
\bibliography{ref}

\appendix

\section{Proofs}\label{appendix:proofs}

\subsection{Proof of Theorem \ref{theorem:linear_equiv_characterization}}\label{proof:theorem_linear_equiv}

Before proving Theorem \ref{theorem:linear_equiv_characterization}, we state and prove a number of helpful properties, propositions and lemmas.

\begin{properties}
    Let $A,B \in \mathcal{A}_{N,d,k,p}$, let $C,S \in \mathcal{A}_{M,d,k',p'}$, let $D,Q \in \mathcal{A}_{N,d,2k+k',p}$, let $\tau \in (\mathbb{Z}/N\mathbb{Z})^d$ be a translation on the $d$-torus, let $\alpha,\beta \in \mathbb{R}$, and let $g \in G_{N,d}$.
    Then the following properties hold.
    \begin{enumerate}
        \item Convolutions are translation equivariant:
        \begin{equation} \label{prop:conv_is_translation_equiv}
            (L_\tau A) * C = L_\tau(A * C) ~.
        \end{equation}
        \item Convolutions are linear in the geometric image:
        \begin{equation} \label{prop:linear_convolution}
            (\alpha A + \beta B) * C = \alpha(A * C) + \beta (B * C) ~.
        \end{equation}
        Convolutions are also linear in the filters:
        \begin{equation} \label{prop:linear_convolution_filters}
            A * (\alpha C + \beta S) = \alpha(A * C) + \beta (A * S) ~.
        \end{equation}
        \item The $k$-contraction is $G_{N,d}$-equivariant:
        \begin{equation} \label{prop:equivariant_contractions}
            g \cdot \contract{k}{D} = \contract{k}{g \cdot D} ~.
        \end{equation}
        \item The $k$-contraction is a linear function:
        \begin{equation}\label{prop:linear_contractions}
            \contract{k}{\alpha D + \beta Q} = \alpha\, \contract{k}{D} + \beta\, \contract{k}{Q} ~.
        \end{equation}
    \end{enumerate}
\end{properties}

\begin{proof}
    First we will prove (\ref{prop:conv_is_translation_equiv}). Let $A,C$, and $\tau$ be as above and let $\bar\imath$ be a pixel index of $L_\tau A * C$. Then:
    \begin{align*}
        (L_\tau A\ast C)(\bar\imath) &= \sum_{\bar a\in[-m, m]^d} (L_\tau A)(\bar\imath - \bar a)\otimes C(\bar a + \bar m) \\
        &= \sum_{\bar a\in[-m, m]^d} A(\bar\imath - \bar a - \tau)\otimes C(\bar a + \bar m) \\
        &= \sum_{\bar a\in[-m, m]^d} A((\bar\imath-\tau) - \bar a)\otimes C(\bar a + \bar m) \\
        &= (A\ast C)(\bar \imath - \tau) \\
        &= L_\tau (A\ast C)(\bar \imath)
    \end{align*}
    Now we will prove (\ref{prop:linear_convolution}). Let $A,B,C,\alpha,$ and $\beta$ be as above and let $\bar\imath$ be a pixel index of $(\alpha A + \beta B) * C$. Then:
    \begin{align*}
        ((\alpha A + \beta B) * C)(\bar\imath) 
        &= \sum_{\bar a \in [-m,m]^d} (\alpha A + \beta B)(\bar \imath - \bar a) \otimes C(\bar a + \bar m) \\
        &= \sum_{\bar a \in [-m,m]^d} (\alpha A(\bar \imath - \bar a) + \beta B(\bar \imath - \bar a)) \otimes C(\bar a + \bar m) \\
        &= \sum_{\bar a \in [-m,m]^d} \alpha A(\bar \imath - \bar a) \otimes C(\bar a + \bar m)  + \beta B(\bar \imath - \bar a) \otimes C(\bar a + \bar m) \\
        &= \alpha \sum_{\bar a \in [-m,m]^d} A(\bar \imath - \bar a) \otimes C(\bar a + \bar m)  +  \beta \sum_{\bar a \in [-m,m]^d} B(\bar \imath - \bar a) \otimes C(\bar a + \bar m) \\
        &= \alpha (A * C)(\bar \imath) + \beta (B * C)(\bar\imath)
    \end{align*}
    Now we will prove \eqref{prop:linear_convolution_filters}. Let $A,C,S,\alpha,$ and $\beta$ be as above and let $\bar\imath$ be a pixel index. Then:
    \begin{align*}
        \qty(A * (\alpha C + \beta S))(\bar\imath) &= 
        \sum_{\bar a \in [-m,m]^d} A(\bar\imath - \bar a) \otimes (\alpha C + \beta S)(\bar a + \bar m) \\
        &= \sum_{\bar a \in [-m,m]^d} A(\bar\imath - \bar a) \otimes \alpha C(\bar a + \bar m) + A(\bar\imath - \bar a) \otimes  \beta S(\bar a + \bar m) \\
        &= \alpha \sum_{\bar a \in [-m,m]^d} A(\bar\imath - \bar a) \otimes C(\bar a + \bar m) + \beta \sum_{\bar a \in [-m,m]^d} A(\bar\imath - \bar a) \otimes S(\bar a + \bar m) \\
        &= \alpha (A*C)(\bar\imath) + \beta (A*S)(\bar\imath)
    \end{align*}
    
    Next we will prove (\ref{prop:equivariant_contractions}). Let $D$ be defined as above and let $\bar\imath$ be a pixel of $D$.
    First we will show that contractions are equivariant to translations. Let $\tau \in (\mathbb{Z}/N\mathbb{Z})^d$. Then 
    \begin{equation}
        \contract{k}{L_\tau D}(\bar \imath) = \contract{k}{(L_\tau D)(\bar\imath)} = \contract{k}{D(\bar\imath - \tau)} = \contract{k}{D}(\bar\imath - \tau) = L_\tau \contract{k}{D}(\bar \imath) ~.
    \end{equation}
    Thus contractions are equivariant to translations. Now we will show that contractions are equivariant to $B_d$. Let $g \in B_d$, and denote $D(g^{-1} \cdot \bar\imath) = a$. Then by equation (\ref{def:tensor_rotation}) we have:
    \begin{align*}
        \contract{k}{g \cdot D}(\bar\imath) &= \contract{k}{(g \cdot D)(\bar\imath)} \\
        &= \contract{k}{g \cdot D(g^{-1} \cdot \bar\imath)} \\
        &= \contract{k}{g \cdot a} \\
        &= [g \cdot a]_{i_1, \ldots, i_k, i_1, \ldots, i_k, i_{2k+1}, \ldots, i_{2k+k'}} \\
        &= [a]_{j_1, \ldots, j_{2k+k'}} \prod_{q \in [k]} [M(g)]_{i_q, j_q} [M(g)]_{i_q, j_{q+k}} \prod_{q \in [2k+1,2k+k']} [M(g)]_{i_q, j_q} \\
        \overset{(*)}&{=} [a]_{j_1, \ldots, j_{2k+k'}} \prod_{q \in [k]} [\delta]_{j_q, j_{q+k}} \prod_{q \in [2k+1,2k+k']} [M(g)]_{i_q, j_q} \\
        &= [a]_{j_1, \ldots, j_k, j_1, \ldots, j_k, j_{2k+1}, \ldots j_{2k+k'}} \prod_{q \in [2k+1,2k+k']} [M(g)]_{i_q, j_q} \\
        &= [\contract{k}{a}]_{j_{2k+1}, \ldots j_{2k+k'}} \prod_{q \in [2k+1,2k+k']} [M(g)]_{i_q, j_q} \\
        &= g \cdot \contract{k}{a} \\
        &= g \cdot \contract{k}{D(g^{-1} \cdot \bar\imath)} \\
        &= (g \cdot \contract{k}{D})(\bar\imath)
    \end{align*}
    Note that $(*)$ happens because $[M(g)]_{i,j}[M(g)]_{i,k} = M(g)^\top M(g) = \delta$ because they are orthogonal matrices, and the next step follows from Kronecker delta identities.
    Therefore, since contractions are equivariant to the generators of $G_{N,d}$, it is equivariant to the group.

    Finally, we will prove (\ref{prop:linear_contractions}). Let $D,Q,\alpha,$ and $\beta$ be defined as above, let $\bar\imath$ be a pixel index of $(\alpha D + \beta Q)$, and let $a,b \in \mathcal{T}_{d,k,p}$ be the tensors of $D$ and $Q$ at that pixel index. Then:
    \begin{align*}
        \qty[\contract{k}{\alpha D + \beta Q}(\bar\imath)]_{i_{2k+1}, \ldots, i_{2k+k'}} &= \qty[\contract{k}{\alpha D(\bar\imath) + \beta Q(\bar\imath)}]_{i_{2k+1}, \ldots, i_{2k+k'}} \\
        &= \qty[\contract{k}{\alpha a + \beta b}]_{i_{2k+1}, \ldots, i_{2k+k'}} \\
        &= \qty[\alpha a + \beta b]_{i_1, \ldots, i_k, i_1, \ldots, i_k,i_{2k+1}, \ldots, i_{2k+k'}} \\
        &= \alpha \qty[a]_{i_1, \ldots, i_k, i_1, \ldots, i_k,i_{2k+1}, \ldots, i_{2k+k'}} + \beta [b]_{i_1, \ldots, i_k, i_1, \ldots, i_k,i_{2k+1}, \ldots, i_{2k+k'}} \\
        &= \alpha \qty[\contract{k}{a}]_{i_{2k+1}, \ldots, i_{2k+k'}} + \beta [\contract{k}{b}]_{i_{2k+1}, \ldots, i_{2k+k'}} \\
        &= \alpha \qty[\contract{k}{D(\bar\imath)}]_{i_{2k+1}, \ldots, i_{2k+k'}} + \beta [\contract{k}{Q(\bar\imath)}]_{i_{2k+1}, \ldots, i_{2k+k'}} \\
        &= \qty[\qty(\alpha \contract{k}{D} + \beta \contract{k}{Q})(\bar\imath)]_{i_{2k+1}, \ldots, i_{2k+k'}}
    \end{align*}
    Thus we have shown (\ref{prop:linear_contractions}).
\end{proof}

\begin{lemma}\label{lemma:even_filter}
    Given $A \in \mathcal{A}_{N,d,k,p}$ a geometric image and $C \in \mathcal{A}_{M,d,k',p'}$ a geometric filter where $M=N+1$, there exists $C' \in \mathcal{A}_{M,d,k',p'}$ such that $A * C' = A * C$ and $C'(\bar \imath)$ is the zero \tensor{k'}{p'}, for $\bar \imath \in [0,N]^d \setminus [0, N-1]^d$. 
    That is, $C'$ is totally defined by $N^d$ pixels, and every pixel with an $N$ in the index is equal to the zero \tensor{k'}{p'}.
\end{lemma}

\begin{proof}
    Let $A$ and $C$ be defined as above. Thus 
    \begin{equation}
        N = M - 1 = 2m + 1 - 1 = 2m
    \end{equation}
    Consider the convolution definition (\ref{def:convolution}) where we have $A(\bar\imath - \bar a)$ where $\bar\imath \in [0,N-1]^d$ and $\bar a \in [-m,m]^d$. Since $A$ is on the $d$-torus, then whenever the $\ell^{th}$ index of $\bar a = -m$ we have:
    \begin{align*}
        \qty(\bar \imath_\ell - \bar a_\ell) \mod N &= \qty(\bar \imath_\ell - (-m)) \mod N \\
        &= \qty(\bar \imath_\ell + m) \mod 2m \\
        &= \qty(\bar \imath_\ell + m - 2m) \mod 2m \\
        &= \qty(\bar \imath_\ell - m) \mod N
    \end{align*}
    Thus, any time there is an index $\bar a$ with a value $\pm m$, we have an equivalence class under the torus with all other indices with flipped sign of the $m$ in any combination. If $\qty{\bar a}$ is this equivalence class, we may group these terms in the convolution sum:
    \begin{equation*}
        \sum_{\bar a' \in \qty{\bar a}} A(\bar \imath - \bar a') \otimes C(\bar a' + \bar m) = \sum_{\bar a' \in \qty{\bar a}} A(\bar \imath - \bar a) \otimes C(\bar a' + \bar m) = A(\bar \imath - \bar a) \otimes \qty(\sum_{\bar a' \in \qty{\bar a}} C(\bar a' + \bar m))
    \end{equation*}
    Thus, we may pick a single pixel of the convolutional filter $C$, set it equal to $\sum_{\bar a' \in \qty{\bar a}} C(\bar a' + \bar m)$, and set all other pixels of the equivalence class to the zero \tensor{k'}{p'} without changing the convolution. We choose the nonzero pixel to be the one whose index has all $-m$ instead of $m$. Thus we can define the filter $C$ by $N^d$ pixels rather than $(N+1)^d$ pixels, and we have our result.
\end{proof}

\begin{lemma}\label{lemma:equality_of_filters}
    Let there be a space of geometric images $\mathcal{A}_{N,d,k,p}$ and let $C_1,C_2 \in \mathcal{A}_{M,d,k+k',p\,p'}$ with $M = 2m+1$ for positive integer $m$.
    Then $\contract{k}{A*C_1} = \contract{k}{A*C_2}$ for all $A \in \mathcal{A}_{N,d,k,p}$ if and only if $C_1 = C_2$.
\end{lemma}

Here is a quick proof sketch of the forward direction. 
We assume for the sake of contradiction that $C_1$ and $C_2$ are different so they must have at least one differing component.
Then we use the fact that $\contract{k}{A*C_1} = \contract{k}{A*C_2}$ holds for all possible inputs to define an input $A$ that isolates that component to get a contradiction.

\begin{proof}
    Let $C_1,C_2$ be defined as above. 
    The reverse direction is immediate, so we focus our attention on the forward direction.
    Suppose $\contract{k}{A*C_1} = \contract{k}{A*C_2}$ for all $A \in \mathcal{A}_{N,d,k,p}$. 
    Assume for the sake of contradiction that $C_1 \not = C_2$, so $C_1 - C_2 \not = \vec{0}$, where $\vec{0}$ is the zero filter.
    Thus there must be at least one component of one pixel that is nonzero.
    Suppose this is at pixel index $\bar b + \bar m$ and $(C_1 - C_2)(\bar b + \bar m) = c$. 
    Suppose the nonzero component is at index $j_1, \ldots, j_{k+k'}$.
    Let $a$ be a \tensor{k}{p} where $[a]_{i_1, \ldots, i_k}$ is nonzero and all other indices are 0. 
    Now suppose $A \in \mathcal{A}_{N,d,k,p}$ such that for pixel index $\bar \imath$ of $A, A(\bar \imath - \bar b) = a$ and all other pixels are the zero tensor. Thus:
    \begin{align*}
        \vec{0} &= (\contract{k}{A*C_1} - \contract{k}{A*C_2})(\bar\imath) \\
        \overset{\ref{prop:linear_convolution_filters},\ref{prop:linear_contractions}}&{=} \contract{k}{A * (C_1 - C_2)}(\bar\imath) \\
        &= \contract{k}{\qty(A * (C_1 - C_2)}(\bar \imath)) \\
        &= \contract{k}{\sum_{\bar a \in [-m,m]^d} A(\bar \imath - \bar a) \otimes (C_1 - C_2)\qty(\bar a + \bar m)} \\
        &= \contract{k}{A(\bar \imath - \bar b) \otimes (C_1 - C_2)(\bar b + \bar m)} \\
        &= \contract{k}{a \otimes c} ~.
    \end{align*}
    Note that the penultimate step removing the sum is because $A(\bar\imath - \bar a) = 0$ the zero tensor everywhere other than $A(\bar\imath - \bar b)$.
    Therefore, since the only nonzero entry of $a$ is at index $i_1, \ldots, i_k$, then at index $j_{k+1},\ldots,j_{k+k'}$ of the resulting tensor we have:
    \begin{align*}
        \vec{0} = \contract{k}{a \otimes c} = [a]_{i_1 \ldots i_k}[c]_{j_1, \ldots,j_{k+k'}} ~.
    \end{align*}
    Since $[a]_{i_1, \ldots, i_k}$ is nonzero and $[c]_{j_1, \ldots, j_{k+k'}}$ is nonzero, this index is nonzero. This is a contradiction, so we conclude that $C_1 = C_2$ which finishes the proof.
\end{proof}

\begin{proposition*}[Restatement of \ref{prop:convolution_iff_translation_invariant}]
    A function $f:\mathcal A_{N,d,k,p}\to \mathcal{A}_{N,d,k',p'}$ is a translation equivariant linear function if and only if $f(A) = \contract{k}{A*C}$ for some geometric filter $C \in \mathcal{A}_{M,d,k+k',p\,p'}$. 
    When $N$ is odd, $M=N$, otherwise $M=N+1$.
\end{proposition*}

\begin{proof}\label{proof:convolution_iff_translation_invariant}
    Let $\mathcal{F} = \qty{ f:\mathcal{A}_{N,d,k,p} \to \mathcal{A}_{N,d,k',p'}}$ where each function $f$ is linear and equivariant to translations. 
    Let $\mathcal{G} = \qty{ g: \mathcal{A}_{N,d,k,p} \to \mathcal{A}_{N,d,k',p'}}$ where each $g$ is defined as $g(A) = \contract{k}{A * C}$ for some $C \in \mathcal{A}_{M,d,k+k',p\,p'}$. 
    If $N$ is odd then $M=N$, otherwise $M=N+1$. It suffices to show that $\mathcal{F} = \mathcal{G}$.

    First we will show that $\mathcal{G} \subseteq \mathcal{F}$. 
    Let $g \in \mathcal{G}$. By properties \eqref{prop:linear_convolution} and \eqref{prop:linear_contractions} both convolutions and contractions are linear. 
    Additionally, by properties \eqref{prop:conv_is_translation_equiv} and \eqref{prop:equivariant_contractions} convolutions and contractions are both equivariant to translations. 
    Thus $g \in \mathcal{F}$, so $\mathcal{G} \subseteq \mathcal{F}$.
    
    Now we will show that $\dim(\mathcal{F}) = \dim(\mathcal{G})$. 
    Let $f \in \mathcal{F}$. 
    By Definition \ref{def:tensor}, $\mathcal{T}_{d,k,p} \cong \qty(\mathbb{R}^d)^{\otimes k}$ equipped with the group action of $O(d)$.
    Then by Definition \ref{def:geometric_image}, $\mathcal{A}_{N,d,k,p}$ is the space of functions $A:[N]^d \to \mathcal{T}_{d,k,p}$ where $[N]^d$ has  the structure of the $d$-torus. 
    Therefore, $\mathcal{A}_{N,d,k,p} \cong \qty(\mathbb{R}^N)^{\otimes d} \times \qty(\mathbb{R}^d)^{\otimes k}$ equipped with the group action of $G_{N,d}$.
    Thus, $f: \qty(\mathbb{R}^N)^{\otimes d} \times \qty(\mathbb{R}^d)^{\otimes k} \to \qty(\mathbb{R}^N)^{\otimes d} \times \qty(\mathbb{R}^d)^{\otimes k'}$. 
    Since $f$ is linear, the dimension of the space of functions $\mathcal{F}$ is $N^dd^{k'}N^dd^k = N^{2d}d^{k+k'}$. 
    If this is unclear, consider the fact that the linearity of $f$ means that it has an associated matrix $F$ of that dimension.
    However, since each $f$ is translation equivariant, the function to each of the $N^d$ pixels in the output must be the same. 
    Thus we actually have that $\dim(\mathcal{F}) = \frac{N^{2d}d^{k+k'}}{N^d} = N^d d^{k+k'}$.

    Now we look at $\dim(\mathcal{G})$.
    Each function $g \in \mathcal{G}$ is defined by the convolution filter $C \in \mathcal{A}_{M,d,k+k',p\,p'}$ and $\dim(\mathcal{A}_{M,d,k+k',p\,p'}) = \dim(\mathcal{A}_{N,d,k+k',p\,p'}) = N^d d^{k+k'}$, with the first equality following from Lemma \ref{lemma:even_filter} in both the even and odd case.
    Clearly $\dim(\mathcal{G})$ is upper-bounded by the dimension of the convolution filters, but does it have to be equal?
    In other words, is it possible that two linearly independent convolution filters result in linearly dependent functions $g$?
    We will now show that this is not possible.

    Let $g_1,g_2 \in \mathcal{G}$ be defined by two linearly independent filters $C_1,C_2 \in \mathcal{A}_{M,d,k+k',p\,p'}$, and we would like to show that $g_1$ and $g_2$ are linearly independent as well.
    Suppose that there exists $\alpha,\beta \in \mathbb{R}$ such that $\alpha\,g_1(A) + \beta\,g_2(A) = \vec{0}$ for all $A \in \mathcal{A}_{N,d,k,p}$.
    It suffices to show that $\alpha = \beta = 0$. Thus:
    \begin{align*}
        \contract{k}{A*\vec{0}} &= \contract{k}{A*(0\,C_3)} \\
        \overset{\ref{prop:linear_convolution_filters},\ref{prop:linear_contractions}}&{=} 0\, \contract{k}{A*C_3} &&  \\
        &= \vec{0} \\
        &= \alpha\, g_1(A) + \beta\, g_2(A) \\
        &= \alpha\, \contract{k}{A*C_1} + \beta\, \contract{k}{A*C_2} \\
        \overset{\ref{prop:linear_convolution_filters},\ref{prop:linear_contractions}}&{=} \contract{k}{A*\qty(\alpha\,C_1 + \beta\,C_2)} ~.
    \end{align*}
    Thus by Lemma \ref{lemma:equality_of_filters}, $\vec{0} = \alpha\, C_1 + \beta\, C_2$. Since $C_1$ and $C_2$ are linearly independent, this implies that $\alpha = \beta = 0$.
    Thus $g_1,g_2$ must be linearly independent.
    Therefore, $\dim(\mathcal{G}) = N^d d^{k+k'}$ and since $\mathcal{G} \subseteq \mathcal{F}$ we have $\mathcal{F} = \mathcal{G}$.
\end{proof}

\begin{lemma}\label{lemma:action_distributes_over_conv}
    Given $g \in B_d$, $A \in \mathcal{A}_{N,d,k,p},$ and $ C \in \mathcal{A}_{M,d,k',p'}$, the action $g$ distributes over the convolution of $A$ with $C$:
    \begin{equation}
        g \cdot (A * C) = (g \cdot A) * (g \cdot C) ~.
    \end{equation}
\end{lemma}

\begin{proof}
    Let $A \in \mathcal{A}_{N,d,k,p}$ be a geometric image, let $C \in \mathcal{A}_{M,d,k',p'}$, let $g \in B_d$, and let $\bar\imath$ be any pixel index of $A$. By Definition \ref{def:action_on_tensor_image} we have
    \begin{align*}
        (g\cdot (A* C) )\qty(\bar \imath) &= g\cdot\qty( (A*C)\qty(g^{-1} \cdot \bar \imath)) \\
        &= g \cdot \qty(\sum_{\bar a\in [-m,m]^d} A\qty(g^{-1}\cdot \bar \imath - \bar a) \otimes C\qty(\bar a + \bar m))\\
        &= \sum_{\bar a\in [-m,m]^d} g \cdot \qty(A\qty(g^{-1}\cdot \bar \imath - \bar a) \otimes C\qty(\bar a + \bar m))\\
        &= \sum_{\bar a\in [-m,m]^d} g \cdot A\qty(g^{-1}\cdot \bar \imath - \bar a) \otimes g \cdot C\qty(\bar a + \bar m)
    \end{align*}
    Now let $\bar a' = g \cdot \bar a$. Thus $g^{-1} \cdot \bar a' = g^{-1} \cdot g \cdot \bar a = \bar a$. Then:
    \begin{align*}
        (g \cdot (A * C))(\bar\imath) &= \sum_{\bar a\in [-m,m]^d}  g \cdot A\qty(g^{-1}\cdot \bar \imath - \bar a) \otimes g \cdot C\qty(\bar a + \bar m) \\
        &= \sum_{g^{-1} \cdot \bar a' \in [-m,m]^d} g \cdot A\qty(g^{-1}\cdot \bar \imath - g^{-1} \cdot \bar a') \otimes g \cdot C\qty(g^{-1} \cdot \bar a' + \bar m) \\
        &= \sum_{g^{-1} \cdot \bar a' \in [-m,m]^d} g \cdot A\qty(g^{-1}\cdot \bar \imath - g^{-1}\cdot \bar a') \otimes g \cdot C\qty(g^{-1}\cdot \bar a' + g^{-1}\cdot \bar m) \\
        &= \sum_{g^{-1} \cdot \bar a' \in [-m,m]^d} g \cdot A\qty(g^{-1}\cdot \qty(\bar \imath -\bar a')) \otimes g \cdot C\qty(g^{-1}\cdot \qty(\bar a' + \bar m)) \\
        &= \sum_{g^{-1} \cdot \bar a' \in [-m,m]^d} \qty(g \cdot A)\qty(\bar\imath - \bar a') \otimes \qty(g \cdot C)\qty(\bar a'+ \bar m) \\
        &= \sum_{\bar a' \in [-m,m]^d} \qty(g \cdot A)\qty(\bar\imath - \bar a') \otimes \qty(g \cdot C)\qty(\bar a'+ \bar m) \\
        &= \qty((g \cdot A) * (g \cdot C))(\bar\imath)
    \end{align*}
    For the penultimate step, we note that $g^{-1} \cdot \bar a' \in [-m,m]^d$ compared to $\bar a' \in [-m,m]^d$ is just a reordering of those indices in the sum. Thus we have our result for pixel $\bar\imath$, so it holds for all pixels.
\end{proof}

Now we will prove Theorem \ref{theorem:linear_equiv_characterization}.

\begin{theorem*}[Restatement of \ref{theorem:linear_equiv_characterization}]
    A function $f:\mathcal{A}_{N,d,k,p} \to \mathcal{A}_{N,d,k',p'}$ is linear and $G_{N,d}$-equivariant if and only if it can be written as $\contract{k}{A * C}$ for some $B_d$-isotropic $C \in \mathcal{A}_{M,d,k+k',p\,p'}$, where $M=N$ if $N$ is even and $M=N+1$ otherwise.
\end{theorem*}

\begin{proof}
    First we will show the reverse direction. Let $C \in \mathcal{A}_{M,d,k+k',p\,p'}$ be $B_d$-isotropic, and let a function $f$ be defined as $f(A) = \contract{k}{A * C}$. Let $g \in B_d, A \in \mathcal{A}_{N,d,k,p}$. Then by the invariance of $C$ we have:
    \begin{align*}
        \contract{k}{(g \cdot A) * C} &= \contract{k}{(g \cdot A) * (g \cdot C)} \\
        \overset{\ref{lemma:action_distributes_over_conv}}&{=} \contract{k}{g \cdot (A * C)} \\
        \overset{\ref{prop:equivariant_contractions}}&{=} g \cdot \contract{k}{A * C} ~.
    \end{align*}
    Hence $f$ is $B_d$-equivariant. 
    By \eqref{prop:conv_is_translation_equiv} and \eqref{prop:equivariant_contractions} $f$ is also translation equivariant, so it is equivariant to $G_{N,d}$. 
    Also, by the linearity of convolution \eqref{prop:linear_convolution} and contraction \eqref{prop:linear_contractions}, $f$ is linear. 
    Thus the reverse direction holds.
    
    Now we will prove the forward direction. Let $f: \mathcal{A}_{N,d,k,p} \to \mathcal{A}_{N,d,k',p'}$ be a linear $G_{N,d}$-equivariant function. 
    Thus $f$ must be translation equivariant, so by Proposition \ref{prop:convolution_iff_translation_invariant} we can write $f$ as $f(A) = \contract{k}{A * C}$ for some $C \in \mathcal{A}_{M,d,k+k',p\,p'}$. 
    Now it suffices to show that $C$ is $B_d$-isotropic.
    Let $A \in \mathcal{A}_{N,d,k,p}$, let $g \in B_d$, and let $B = g ^{-1} \cdot A$. Then by the equivariance of $f$ we have:
    \begin{align*}
         \contract{k}{A * C} &= \contract{k}{\qty(g \cdot B) * C} \\
        &= g \cdot \contract{k}{B*C}  \\
        \overset{\ref{prop:equivariant_contractions}}&{=} \contract{k}{g \cdot \qty(B * C)} \\
        \overset{\ref{lemma:action_distributes_over_conv}}&{=} \contract{k}{(g \cdot B) * (g \cdot C)} \\
        &= \contract{k}{A * (g \cdot C)} ~.
    \end{align*}
    Thus by Lemma \ref{lemma:equality_of_filters}, we have $g \cdot C = C$.
    Therefore, $C$ is $B_d$-isotropic, and this completes the proof.
\end{proof}

\subsection{Extension of Vector Neuron Nonlinearities to Tensors}\label{appendix:nonlinear}

We extend the vector neuron nonlinearities of \cite{deng2021vectorneuronsgeneralframework} for any tensor as follows.

\begin{definition}
    Let $A_i \in \mathcal{A}_{N,d,k,p}$ for $i=1,\ldots,c$ be $c$ channels of input geometric images. 
    Let $\alpha_i, \beta_i \in \mathbb{R}$ for $i=1,\ldots,c$ be learned scalar parameters, and $Q = \sum_{i=1}^c \alpha_i \, A_i, K=\sum_{i=1}^c \beta_i \, A_i$.
    Then the nonlinearity $\sigma: \qty(\mathcal{A}_{N,d,k,p})^c \to \mathcal{A}_{N,d,k,p}$ is defined:
    \begin{equation}
        \sigma(\qty(A_i)_{i=1}^c) = \begin{cases}
            Q & \text{ if } \contract{k}{Q \otimes K} \geq 0 \\
            Q - \contract{k}{Q \otimes \frac{K}{\norm{K}_2}} \frac{K}{\norm{K}_2} & \text{ otherwise } 
        \end{cases}
    \end{equation}
    where $\norm{\cdot}_2$ is the tensor norm \eqref{eq:frobenius_norm}.
\end{definition}

To get $c$ output channels, we can repeat this function $c$ times with different learned parameters $\alpha_i,\beta_i$.
Now we can show that this function is $G_{N,d}$-equivariant.

\begin{proposition} \label{prop:vn_nonlinear_equiv}
    Let $A_i \in \mathcal{A}_{N,d,k,p}$, $g \in G_{N,d}$, and $\alpha_i,\beta_i \in \mathbb{R}$ for $i=1,\ldots,c$. 
    Then $\sigma\qty(\qty(g \cdot A_i)_{i=1}^c) = g \cdot \sigma\qty(\qty(A_i)_{i=1}^c)$.
\end{proposition}

\begin{proof}
    Let $A_i \in \mathcal{A}_{N,d,k,p},$ and $\alpha_i,\beta_i \in \mathbb{R}$ for $i=1,\ldots,c$.
    It is clear to see that $\sigma$ is translation equivariant because all the operations are pixelwise.
    Thus we will show that $\sigma$ is equivariant to $g \in B_d$.
    First, note that applying $g$ to all $A_i$ results in $g \cdot Q$ and $g \cdot K$.
    Now
    \begin{equation*}
        \contract{k}{g \cdot Q \otimes g \cdot K} = \contract{k}{g \cdot \qty(Q \otimes K)} = g \cdot \contract{k}{Q \otimes K} = \contract{k}{Q \otimes K}
    \end{equation*}
    Note that the last step is because both $Q$ and $K$ are \tensors{k}{p}, so $\contract{k}{Q \otimes K}$ is a \tensor{0}{+}.
    Hence, if $\contract{k}{Q \otimes K} \geq 0$, then $\sigma\qty(\qty(g \cdot A_i)_{i=1}^c) = g \cdot Q = g \cdot \sigma\qty(\qty(A_i)_{i=1}^c)$ and $\sigma$ is $B_d$-equivariant.
    Now suppose $\contract{k}{Q \otimes K} < 0$: 
    \begin{align*}
        \sigma(\qty(g \cdot A_i)_{i=1}^c) &= g \cdot Q - \contract{k}{g \cdot Q \otimes \frac{g \cdot K}{\norm{g \cdot K}_2}} \frac{g \cdot K}{\norm{g \cdot K}_2} \\
        &= g \cdot Q - \contract{k}{g \cdot Q \otimes g \cdot \frac{K}{\norm{K}_2}} g \cdot \frac{K}{\norm{K}_2} \\
        &= g \cdot Q - g \cdot \qty(\contract{k}{Q \otimes \frac{K}{\norm{K}_2}} \frac{K}{\norm{K}_2}) \\
        &=  g \cdot \qty(Q - \contract{k}{Q \otimes \frac{K}{\norm{K}_2}} \frac{K}{\norm{K}_2}) \\
        &= g \cdot \sigma\qty(\qty(A_i)_{=1}^c)
    \end{align*}
    Thus $\sigma$ is $B_d$-equivariant.
\end{proof}

\subsection{LayerNorm equivariance}\label{proof:layer_norm_equivariance}

We define equivariant layer norm as the following

\begin{definition}\label{def:layer_norm}
    Let $(A_i)_{i=1}^c$ be a set of \tensor{1}{p} images. 
    Let $\bar{A}_i(\bar\imath) = A_i(\bar\imath) - \sum_{j=1}^c \sum_{\bar\jmath \in [N]^d} A_j(\bar\jmath)$ be the mean centered \tensor{1}{p} image.
    Then the covariance is a \tensor{2}{+} given by 
    \begin{equation}\label{eq:layer_norm}
        \Sigma = \frac{1}{c N^d}\sum_{i=1}^c \sum_{\bar\imath \in [N]^d} \qty(\bar{A}_i \otimes \bar{A}_i)(\bar\imath) ~.
    \end{equation}
    We calculate $\Sigma^{-\frac{1}{2}}$ by performing an eigenvalue decomposition $\Sigma = U \Lambda U^\top$, where $\Lambda$ is a diagonal matrix with the eigenvalues along the diagonal.
    We take the inverse of each eigenvalue and then its square root, then multiply $U \Lambda^{-\frac{1}{2}}U^\top$ to get $\Sigma^{-\frac{1}{2}}$.
    Finally, we scale the vectors by $\Sigma^{-\frac{1}{2}}$:
    \begin{equation}
        [B_i(\bar\imath)]_\ell = \qty[\bar{A}_i(\bar\imath)]_j \qty[\Sigma^{-\frac{1}{2}}]_{j,\ell} ~,
    \end{equation}
    and output $B_i$ for $i=1$ to $c$.
\end{definition}

\begin{proposition}
    The LayerNorm is $G_{N,d}$-equivariant.
\end{proposition}

\begin{proof}
    Let $(A_i)_{i=1}^c$ be a set of \tensor{1}{p} images. Clearly this function will be translation equivariant because  Let $g \in B_d$. Let $\bar{A}_i$ be as defined in Definition \ref{def:layer_norm} and let $\bar\imath$ be a pixel index of $\bar{A}_i$. 
    Then 
    \begin{align}
        (g \cdot \bar{A}_i)(\bar\imath) &= g \cdot \bar{A}_i(g^{-1} \cdot \bar\imath) \\
        &= g \cdot \qty(A_i(g^{-1} \cdot \bar\imath) - \sum_{j=1}^c \sum_{\bar\jmath \in [N]^d} A_j(\bar\jmath)) \\
        &= g \cdot A_i(g^{-1} \cdot \bar\imath) - \sum_{j=1}^c \sum_{\bar\jmath \in [N]^d} g \cdot A_j(\bar\jmath) \\
        &= g \cdot A_i(g^{-1} \cdot \bar\imath) - \sum_{j=1}^c \sum_{g^{-1} \cdot \bar\jmath \in [N]^d} g \cdot A_j(g^{-1} \cdot \bar\jmath) \label{eq:layer_norm_sum_equality} \\
        &= (g \cdot A_i)(\bar\imath) - \sum_{j=1}^c \sum_{\bar\jmath \in [N]^d} (g \cdot A_j)(\bar\jmath) ~.
    \end{align}
    Note that \ref{eq:layer_norm_sum_equality} follows because $\sum_{\bar\jmath \in [N]^d} A_j(\bar\jmath) = \sum_{g^{-1} \cdot \bar\jmath \in [N]^d} A_j(g^{-1} \cdot \bar\jmath)$.
    Thus the mean centering is equivariant to $B_d$. 
    Likewise,
    \begin{align}
        g \cdot \Sigma &= g \cdot \frac{1}{cN^d} \sum_{i=1}^c \sum_{\bar\imath \in [N]^d} (\bar{A}_i \otimes \bar{A}_i)(\bar\imath) \\
        &= \frac{1}{cN^d} \sum_{i=1}^c \sum_{\bar\imath \in [N]^d} g \cdot (\bar{A}_i \otimes \bar{A}_i)(\bar\imath) \\
        &= \frac{1}{cN^d} \sum_{i=1}^c \sum_{g^{-1} \cdot \bar\imath \in [N]^d} g \cdot (\bar{A}_i \otimes \bar{A}_i)(g^{-1} \cdot \bar\imath) \\
        &= \frac{1}{cN^d} \sum_{i=1}^c \sum_{\bar\imath \in [N]^d} (g \cdot \bar{A}_i \otimes g \cdot \bar{A}_i)(\bar\imath)
    \end{align}
    Finally, the inverse square root operation is $B_d$-equivariant.
    If we write it as the function $f$ such that $f(\Sigma) = f(U\Lambda U^\top) = U \Lambda^{-\frac{1}{2}} U^\top = \Sigma^{-\frac{1}{2}}$. Then:
    \begin{align}
        g \cdot f\qty(\Sigma) &= g \cdot f\qty(U \Lambda U^\top) \\
        &= M(g) U \Lambda^{-\frac{1}{2}} U^\top M(g)^\top \\
        &= \qty(M(g) U) \Lambda^{-\frac{1}{2}} \qty(M(g) U )^\top \\
        &= f\qty(\qty(M(g) U) \Lambda \qty(M(g) U )^\top) \\
        &= f(g \cdot \Sigma)
    \end{align}
    Thus $[g \cdot B_i(\bar\imath)]_\ell = \qty[g \cdot \bar{A}_i(\bar\imath)]_\ell \qty[g \cdot \Sigma^{-\frac{1}{2}}]_{j,\ell}$ which is the same as rotating all the input $A_i$, so LayerNorm is equivariant.
\end{proof}

\subsection{Max pool equivariance}\label{proof:max_pool_equivariance}

The $O(d)$-invariance of the tensor norm allows the max pool layer to be $B_d$-equivariant.
With a careful definition of translations for the larger and smaller images, we can also get translational equivariance, as we see in the following proposition.

\begin{proposition}\label{prop:max_pool_equivariance}
    Let $g \in B_d$ and let $\tau \in (\mathbb{Z}/(N/b)\mathbb{Z})^d$ be the translation on the $d$-torus with sidelengths of $N/b$. 
    For images $A \in \mathcal{A}_{N,d,k,p}$, we define the action of this translation as $(L_\tau A)(\bar\imath) = A(\bar\imath - b \, \tau)$.
    Then $\text{max\,pool}_b$ \eqref{eq:max_pool} is equivariant to both of these groups.
\end{proposition}

Before we prove this proposition, we need a quick lemma about the tensor norm \ref{def:frobenius_norm}

\begin{lemma}\label{lemma:frobenius_norm_invariant}
    The tensor norm \eqref{eq:frobenius_norm} is $O(d)$-invariant.
\end{lemma}

\begin{proof}
    Let $c$ be a \tensor{k}{p} and let $g \in O(d)$. Then:
    \begin{equation}
        \norm{g \cdot c}_2 = \sqrt{\contract{k}{g \cdot c \otimes g \cdot c)}} \overset{\eqref{prop:equivariant_contractions}}{=} \sqrt{g \cdot \contract{k}{c \otimes c}} \overset{(*)}{=} \sqrt{\contract{k}{c \otimes c}} = \norm{c}_2
    \end{equation}
    The $(*)$ equality is because $\contract{k}{c \otimes c}$ is always a scalar.
    This completes the proof.
\end{proof}

Now we will prove the proposition.

\begin{proof}
    First we will show equivariance to translations.
    Let $\tau \in (\mathbb{Z}/(N/b)\mathbb{Z})^d$ be the translation on the $d$-torus with sidelengths of $N/b$ as defined in the proposition..
    Let $A \in \mathcal{A}_{N,d,k,p}$ and let $\bar\imath$ be a pixel index.
    Then following the definitions we have:
    \begin{align*}
        &\qty(L_\tau\, \text{max\,pool}_b(A))(\bar\imath) \\
        =&\, \text{max\,pool}_b(A)(\bar\imath - \tau)\\
        =&\, A\qty(b\,(\bar\imath-\tau) + \argmax_{\bar a \in [0, b-1]^d} \norm{A(b\,(\bar\imath - \tau) + \bar a)}_2) \\
        =&\, A\qty(b\,\bar\imath-b\,\tau + \argmax_{\bar a \in [0, b-1]^d} \norm{A(b\,\bar\imath - b\,\tau + \bar a)}_2) \\
        =&\, \qty(L_\tau A)\qty(b\,\bar\imath + \argmax_{\bar a \in [0, b-1]^d} \norm{\qty(L_\tau A)(b\,\bar\imath + \bar a)}_2) \\
        =&\, \text{max\,pool}_b(L_\tau A)(\bar\imath) ~.
    \end{align*}
    Thus $\text{max\,pool}_b$ is equivariant to translations.
    Now let $g \in B_d$. Thus by Lemma \ref{lemma:frobenius_norm_invariant} we have:
    \begin{align*}
        &\qty(g \cdot \text{max\,pool}_b(A))(\bar\imath) \\
        =&\, g \cdot \text{max\,pool}_b(A)(g^{-1} \cdot \bar\imath) \\
        =&\, g \cdot A\qty(b\,(g^{-1} \cdot \bar\imath) + \argmax_{\bar a \in [0, b-1]^d} \norm{A\qty(b\,(g^{-1} \cdot \bar\imath ) + \bar a)}_2) \\
        =&\, g \cdot A\qty(g^{-1} \cdot \qty(b\,\bar\imath + g \cdot \argmax_{\bar a \in [0, b-1]^d} \norm{A\qty(g^{-1}\cdot \qty(b\,\bar\imath + g \cdot \bar a))}_2)) \\
        \overset{\ref{lemma:frobenius_norm_invariant}}{=}&\, (g \cdot A)\qty(b\,\bar\imath + g \cdot \argmax_{\bar a \in [0, b-1]^d} \norm{g \cdot A\qty(g^{-1}\cdot \qty(b\,\bar\imath + g \cdot \bar a))}_2) \\
        =&\, (g \cdot A)\qty(b\,\bar\imath + g \cdot \argmax_{\bar a \in [0, b-1]^d} \norm{(g \cdot A)\qty(b\,\bar\imath + g \cdot \bar a)}_2) \\
        \overset{(*)}{=}&\, (g \cdot A)\qty(b\,\bar\imath + g \,g^{-1} \cdot \argmax_{\bar a \in [0, b-1]^d} \norm{(g \cdot A)\qty(b\,\bar\imath + \bar a)}_2) \\
        =&\, (g \cdot A)\qty(b\,\bar\imath + \argmax_{\bar a \in [0, b-1]^d} \norm{(g \cdot A)\qty(b\,\bar\imath + \bar a)}_2) \\
        =&\, \text{max\,pool}_b(g \cdot A)(\bar\imath) ~.
    \end{align*}
    For the $(*)$ equality we note that $\argmax_{\bar a} \norm{A(g \cdot \bar a)}_2 = g^{-1} \cdot \argmax_{\bar a} \norm{A(\bar a)}_2$ because the pixel index returned by the left side would have to be transformed by $g$ to maximize $\norm{A(\bar a)}_2$.
    Hence the max pool is $B_d$-equivariant, and this concludes the proof.
\end{proof}

\section{Mathematical details of related work}\label{appendix_sec:related}
The most common method to design equivariant maps is via group convolution, on the group or on the homogeneous space where the features lie. Regular convolution of a vector field $f: (\mathbb{Z} /N \mathbb{Z})^d \to \mathbb{R}^c$ and a filter $\phi: (\mathbb{Z} /N \mathbb{Z})^d \to \mathbb{R}^c$ is defined as
\begin{equation}\label{eq:1}
    (f*\phi)(x)= \sum_{y\in(\mathbb{Z} /N \mathbb{Z})^d} \underbrace{\langle f(y), \phi(x-y)\rangle}_{\text{scalar product of vectors}}=\sum_{y\in(\mathbb{Z} /N \mathbb{Z})^d} \sum_{j=1}^c \underbrace{f^j(y) \phi^j(x-y)}_{\in \mathbb{R}}
\end{equation}
Our generalization of convolution replaces this scalar product of vectors by the outer product of tensors. 

\subsection{Clifford Convolution}
Probably the most related
work is by Brandstetter et al. \cite{ref1}, which replaces the scalar product in \eqref{eq:1} by the geometric product of multivector inputs and multivector filters of a Clifford Algebra. It considers multivector fields, i.e.: vector fields $f:\mathbb{Z}^2 \to (Cl_{p,q}(\mathbb{R}))^c$. The real Clifford Algebra $Cl_{p,q}(\mathbb{R})$ is an associative algebra generated by $p+q=d$ orthonormal basis elements: $e_1, \ldots, e_{p+q}\in \mathbb{R}^d$ with the relations:
\begin{eqnarray}
     e_i\otimes e_i = +1& (i\leq p),  \\
     e_j \otimes e_j = -1& (p<j\leq n),\\
     e_i \otimes e_j = -e_j \otimes e_i & (i\neq j).
\end{eqnarray}
For instance, $Cl_{2,0}(\mathbb{R})$ has the basis $\{1, e_1, e_2, e_1\otimes e_2\}$ and is isomorphic to the quaternions $\mathbb{H}$.

The Clifford convolution replaces the elementwise product of scalars of the usual convolution of \eqref{eq:1} by the geometric product of multivectors in the Clifford Algebra:
\begin{equation}
    f*\phi(x)= \sum_{y\in(\mathbb{Z} /N \mathbb{Z})^d} \sum_{j=1}^c \underbrace{f^j(y) \otimes \phi^j(y-x)}_{\in Cl_{p,q}(\mathbb{R})},
\end{equation}
where $f:\mathbb{Z}^2 \to (Cl_{p,q}(\mathbb{R}))^c$ and $\phi:\mathbb{Z}^2 \to (Cl_{p,q}(\mathbb{R}))^c$

The Clifford Algebra $Cl_{p,q}(\mathbb{R})$  is a quotient of the tensor algebra \begin{equation}
    T(\mathbb{R}^d)= \bigoplus_{k\geq 0} \underbrace{\mathbb{R}^d\otimes \ldots \otimes {\mathbb{R}^d}}_{k \text{ times}}= \bigoplus_{k\geq 0} (\mathbb{R}^d)^{\otimes k},
\end{equation} 
by the two-side ideal $\langle\{v\otimes v - Q(v): v \in \mathbb{R}^d\}\rangle$, where the quadratic form $Q$ is defined by $Q(e_i)=+1$,if  $i\leq p$, and $Q(e_j)=-1$, else $p<j\leq n$.
Our geometric images are functions $A:(\mathbb{Z}/N\mathbb{Z})^d \to \mathcal T_{d,k,p}$, where $\mathcal{T}_{d,k,p}= (\mathbb{R}^d)^{\otimes k}\subset T(\mathbb{R}^d)$. They can be related with the Clifford framework by seeing them as $N$-periodic functions from $\mathbb{Z}^d$ whose image is projected via the quotient map on the Clifford Algebra. This projection can be seen as a contraction of tensors. 

The Clifford convolution is not equivariant under multivector rotations or reflections. But the authors derive a constraint on the filters for $d=2$ which allows to build generalized Clifford convolutions which are equivariant with respect to rotations or reflections of the multivectors. That is, they prove equivariance of a Clifford layer under orthogonal transformations if the filters satisfies the constraint: $\phi^i (Rx) = R\phi^i(x)$.
 
\subsection{Unified Fourier Framework}
Part of our work can be studied under the unified framework for group equivariant networks on homogeneous spaces derived from a Fourier perspective proposed in \cite{ref2}. The idea is to consider general tensor-valued feature fields, before and after a convolution.  Their fields are functions $f : G/H \to V$ over the homogeneous space $G/H$ taking values in the vector space $V$ and their filters are kernels $\kappa: G \to Hom (V,V')$. Essentially, their convolution replaces the scalar product of vectors of traditional convolution by appliying an homomorphism. In particular, if $G$ is a finite group and $H=\{0\}$, they define convolution as
\begin{equation}
    \kappa* f(x)= \frac{1}{|G|}\sum_{y\in G}\underbrace{\kappa(x^{-1}\,y)\,f(y)}_{\in V'}.  
\end{equation}
\cite{ref2} gives a complete characterization of the space of kernels for equivariant convolutions. In our framework, the group is $\mathbb{Z}/N\mathbb{Z}$ and the kernel is an outer product by a filter $C$: $\kappa(g)A(g)=A(g)\otimes C(g)$. Note that $\mathbb{Z}/N\mathbb{Z}$ is neither a homogeneous space of $O(d)$ nor of $B^d$.

We can analyze our problem from a spectral perspective, in particular we can describe all linear equivariant using representation theory, using similar tools as in the proof of Theorem 1 in \cite{kondor2018convolution}. This theorem states that convolutional structure is a sufficient and a necessary condition for equivariance to the action of a compact group. Some useful references about group representation theory are \cite{ref5}, a classical book about the theory of abstract harmonic analysis and \cite{ref6}, about the particular applications of it.

\subsection{Linear equivariant maps}
In this work we define an action  over tensor images of $O(d)$, by rotation of tensors in each pixel; of $B^d$ by rotating the grid of pixels and each tensor in the pixel; and of $(\mathbb{Z}/N \mathbb{Z})^d$ by translation of the grid of pixels. The action of each one of these groups $G$ over $\mathcal{T}_{d,k,p}$
\begin{equation}
    \Phi_{d,k,p}: G \to GL_{con} (\mathcal{T}_{d,k,p}),
\end{equation}
can be decomposed into irreducible representations of $G$:
\begin{equation}
    \Phi_{d,k,p} \equiv \bigoplus_{\pi \in \hat{G}} m_{d,k,p}(\pi) \,\pi.
\end{equation}
That is, there is a basis of the Hilbert space $\mathcal{T}_{d,k,p}$ in which the action of $G$ is defined via a linear sparse map. In the case of $G$ finite, for all $g\in G$ there is a matrix $P$ splitting the representation in the Hilbert space into its irreducible components
\begin{equation}
    P^{-1}\, \Phi_{d,k,p}(g)\,P = \bigoplus_{\pi \in \hat{G}} m_{d,k,p}(\pi)\, \pi(g)
\end{equation}

Consider now linear maps between Tensor images:
\begin{equation}
    \mathcal{C}:\mathcal{T}_{d,k,p} \to \mathcal{T}_{d',k',p'}
\end{equation}
Linear equivariant maps satisfy that  $\mathcal{C}\circ \Phi_{d,k,p} = \Phi_{d',k',p'} \circ \mathcal{C}$. That is, if $\tilde{\mathcal{C}}$ is the representation of $\mathcal{C}$ in the above basis,
\begin{equation}
    \tilde{\mathcal{C}} \circ \bigoplus_{\pi \in G} m_{d,k,p}(\pi)\, \pi = \bigoplus_{\pi \in G} m_{d',k',p'}(\pi)\, \pi \circ \tilde{\mathcal{C}}.
\end{equation}
By Schur's Lemma, this implies that $\mathcal{C}\equiv \bigoplus_{\pi \in G} m_{d,k,p}(\pi)\, Id_{d_\pi}$.

The power of representation theory is not limited to compact groups. Mackey machinery allow us to study for instance semidirect products of compact groups and other groups, and in general to relate the representations of a normal subgroup with the ones of the whole group. This is the spirit of \cite{steerable}, which makes extensive use of the induced representation theory. An introduction to this topic can be found in Chapter 7 in \cite{ref5}.

\subsection{Steerable CNNs}
The work in \cite{steerable} deals exclusively with signals $f: \mathbb{Z}^2 \to \mathbb{R}^k$. 
They consider the action of $G=$ \textit{p4m} on $\mathbb{Z}^2$ by translations, rotations by 90 degrees around any point, and reflections. 
This group is a semidirect product of $\mathbb{Z}^2$ and $B_2$, so every $x\in \text{\textit{p4m}}$ can be written as $x=t\,r$, for $t\in \mathbb{Z}^2$ and $r\in B_2$. 
They show that equivariant maps with respect to representations $\rho$ and $\rho'$ of rotations and reflections $B_2$ lead to equivariant maps with respect to certain representations of $G$, $\pi$ and $\pi'$. 
This means that if we find a linear map $\phi: f \mapsto \phi\, f$ such that $\phi \,\rho(h)\,f=\rho'(h)\,\phi \,f$ for all $h\in B_2$, then for the representation of $G$ $\pi'$ defined by
\begin{equation}\label{eq:2}
    \pi'(t\,r)\,f(y) = \rho(r)\, [f ((t\,r)^{-1}\,y)], \quad t\,r\in G, y \in \mathbb{Z}^2,
\end{equation}
we automatically have that $\phi\, \pi(g)\,f=\pi'(g)\,\phi\, f$ for all $g\in G$. This is the representation of $G$ induced by the representation $\rho$ of $B_2$

Note the similarity between the definition of the action of $B_d$  on tensor images \ref{def:action_bd_on_tensor} and equation \eqref{eq:2}. The convolution with a symmetric filter produces easily an equivariant map with respect to  the action of the semidirect product of $\mathbb{Z}^d$ and $B_d$ on the tensor images. 

\subsection{Approximate symmetries}
The recent work \cite{wang2022approximately} studies approximately equivariant networks which are biased towards preserving symmetry but are not strictly constrained to do so. They define a relaxed group convolution which is approximately equivariant in the sense that
\begin{equation}
    \|\rho_X(g)\, f*_G \,\Psi (x) - f*_G\,\Psi(\rho_Y(y)\,x \|<\epsilon.
\end{equation}
They use a  classical convolution but with different kernels for different group elements. 

\section{Experimental Details}\label{appendix:experiments}

\subsection{Data}

The data is the PDEBench files \verb|2D_CFD_Rand_M0.1_Eta0.01_Zeta0.01_periodic_128_Train.hdf5| and \verb|2D_CFD_Rand_M1.0_Eta0.1_Zeta0.1_periodic_128_Train.hdf5| which can be found at \url{https://darus.uni-stuttgart.de/dataset.xhtml?persistentId=doi:10.18419/darus-2986} \cite{takamoto2024pdebench}.
We used the first $128$ trajectories as training data, the next $32$ trajectories as a validation set, and the next $128$ trajectories as a test data set.
The density and pressure fields are mean-centered and scaled to have variance 1 based on the training and validation datasets.
The velocity field is not mean-centered because the only rotationally isotropic vector is the zero vector, but it is scaled to have variance 1 in the components.

\subsection{Models}

Model specifics are described below. 
For equivariant models, we always use ReLU for scalars and the Vector Neuron activation for non-scalars.
For equivariant encoder and decoder blocks, we use $3 \times 3$ filters instead of $1 \times 1$ filters because for some order and parity pairs, there are no $1 \times 1$ $B_d$-isotropic filters.
All convolutions use biases except for the UNet.
For equivariant models, the bias is a scale of the mean tensor of that image.
Additional details are in Table \ref{tab:models}.

\begin{itemize}
    \item \textbf{Dilated ResNet \cite{stachenfeld2022learned}:} 
    The model starts with two ``encoder" convolutions with $1 \times 1$ filters and ReLU activations.
    There are four blocks, each consisting of seven convolutions with dilations of $1,2,4,8,4,2,1$ with associated ReLU activations.
    There are residual connections connecting each block.
    The model concludes with two ``decoder" convolutions with $1 \times 1$ filters and a ReLU activation between the two.
    \item \textbf{ResNet \cite{he2015resnet}:} This model consists of 8 blocks of 2 convolutions each with residual connections between each block. 
    Each block also has LayerNorm and a GeLU activation \cite{hendrycks2023gelu}. 
    We put the LayerNorm and activation prior to the convolution (preactivation order \cite{he2016preactivation_order}) following \cite{gupta2022pdearena}.
    This model also uses two ``encoder" $1 \times 1$ convolutions and two ``decoder" $1 \times 1$ convolutions.
    \item \textbf{UNet LayerNorm \cite{gupta2022pdearena}:} This model is referred to as ``UNetBase" in \cite{gupta2022pdearena}.
    This starts with an embedding block with a convolution with a $3 \times 3$ filter following by LayerNorm and a GeLU activation \cite{hendrycks2023gelu}.
    Next comes a $\text{max\,pool}_2$ followed by two convolutions with LayerNorm and GeLU activation.
    This is process is repeated for 4 total downsamples, and notably the number of convolution channels is doubled for every down sample.
    Then the process happens in reverse, with max pooling replaced with transposed convolution to double the spatial size instead of halving it each time.
    See \cite{dumoulin2018guideconvolutionarithmeticdeep} for a description of transposed convolution.
    The number of convolution channels is also halved each time we upsample.
    The final kicker is that there are also residual connections from before each downsample to after each upsample for the appropriate spatial size.
    The model concludes with a final convolution.
    In the equivariant model we do not include the LayerNorm because it hurt the performance.
    \item \textbf{UNet \cite{ronneberger2015unet}:} This model is the same as the one above, except is uses BatchNorm instead of LayerNorm and the convolutions are without biases.
\end{itemize}

\begin{table}[h]
    \begin{tabular}{c|c|c|c|c|c}
        model & params & CNN channels & norm & bias & learning rate \\
        \hline
        DilResNet & 1,043,651  & 64 & - & Yes & \verb|2e-3| \\
        DilResNet Equiv & 979,347 & 48 & - & Mean & \verb|1e-3| \\
        ResNet & 2,401,155  & 128 & LayerNorm & Yes & \verb|1e-3| \\
        ResNet Equiv & 2,558,703  & 100 & LayerNorm & Mean & \verb|7e-4| \\
        UNet LayerNorm & 31,053,251 & 64 & LayerNorm & Yes & \verb|8e-4| \\
        UNet LayerNorm Equiv & 27,077,139 & 48 & - & Mean & \verb|4e-4| \\
        UNet & 31,046,400 & 64 & BatchNorm & No & \verb|8e-4| \\
        UNet Equiv & 27,066,864 & 48 & - & No & \verb|3e-4| \\
    \end{tabular}
    \caption{Comparison of various models. The number of channels of each model was chosen so that the equivariant and non-equivariant models have roughly the same number of parameters.}
    \label{tab:models}
\end{table}

\subsection{Training}

For a loss function, we use the sum of mean squared error loss, or SMSE.
This loss sums over the tensor components and the channels and takes the mean over the spatial components.
If $\qty{A_i}_{i=1}^c$ are the true \tensor{k_i}{p_i} images and $\qty{\hat{A}_i}_{i=1}^c$ are our predicted \tensor{k_i}{p_i} images, then the $\mathcal{L}_{\text{smse}}$ is defined as,
\begin{equation}
    \mathcal{L}_{\text{smse}}\qty(\qty{A_i}_{i=1}^c, \qty{\hat{A}_i}_{i=1}^c) = \sum_{i=1}^c \frac{1}{N^d} \sum_{\bar\imath} \norm{A_i(\bar\imath) - \hat{A}_i(\bar\imath)}_2^2 ~,
\end{equation}
where $\norm{\cdot}_2$ is the tensor norm.
When calculating a rollout loss, we simply sum the loss of each rollout step.

We follow a similar training regime as in \cite{gupta2022pdearena}.
We train for 50 epochs using the AdamW optimizer \cite{loshchilov2018adamw} with a weight decay of \verb|1e-5| and a cosine decay schedule \cite{loshchilov2017cosinedecay} with 5 epochs of warmup.
Learning rates were tuned for each model, searching for values between \verb|1e-4| and \verb|2e-3|, and are included in Table \ref{tab:models}.

We trained on 4 RTX A5000 graphics cards with a batch size of 8, for an effective batch size of 32.
Experiments we averaged over 3 trials, using the same training data each time.
It possible that different optimizers, learning rate schedules, batch sizes, or other hyperparameters may perform better on the task, but we held those fixed and only tuned the learning rate since our focus is on comparing the equivariant and non-equivariant models.

\end{document}